\documentclass[11pt]{article}
\pdfoutput=1
\usepackage[letterpaper,margin=1in]{geometry}
\usepackage{enumerate}
\usepackage[T1]{fontenc}
\usepackage{microtype}
\usepackage{titlesec}
\titleformat*{\paragraph}{\bfseries}

\usepackage{amsthm,amsmath,amssymb,amsfonts,amssymb,mathtools}
\usepackage{amsmath,amssymb,amsfonts,amssymb,mathtools}
\usepackage{xcolor}
\usepackage{empheq}
\usepackage{dsfont}
\usepackage{mathrsfs}
\usepackage{graphicx}
\usepackage{aliascnt} 
\usepackage{xspace}
\usepackage{bbm}
\usepackage{caption}
\usepackage{tikz}
\usetikzlibrary{patterns}
\usetikzlibrary{patterns}
\usetikzlibrary{arrows,shapes,automata,backgrounds,petri,positioning}
\usetikzlibrary{shadows}
\usetikzlibrary{calc}
\usetikzlibrary{spy}
\usetikzlibrary{angles, quotes}
\usetikzlibrary{matrix}
\usepackage{pgf,pgfplots}
\usepackage{pgfmath,pgffor}
\pgfplotsset{compat=1.17}

\usepackage[shortlabels]{enumitem}
\usepackage{framed}
\usepackage[noend]{algorithm2e}
\usepackage[noend]{algpseudocode}

\definecolor[named]{ACMBlue}{cmyk}{1,0.1,0,0.1}
\definecolor[named]{ACMYellow}{cmyk}{0,0.16,1,0}
\definecolor[named]{ACMOrange}{cmyk}{0,0.42,1,0.01}
\definecolor[named]{ACMRed}{cmyk}{0,0.90,0.86,0}
\definecolor[named]{ACMLightBlue}{cmyk}{0.49,0.01,0,0}
\definecolor[named]{ACMGreen}{cmyk}{0.20,0,1,0.19}
\definecolor[named]{ACMPurple}{cmyk}{0.55,1,0,0.15}
\definecolor[named]{ACMDarkBlue}{cmyk}{1,0.58,0,0.21}

\usepackage[colorlinks,citecolor=blue,linkcolor=magenta,bookmarks=true]{hyperref}
\usepackage[nameinlink]{cleveref}
\crefname{ineq}{Inequality}{Inequality}
\creflabelformat{ineq}{#2{\upshape(#1)}#3}
\crefname{sub}{Subsection}{Subsection}
\creflabelformat{Subsection}{#2{\upshape(#1)}#3}
\crefname{sdp}{SDP}{SDP}
\creflabelformat{sdp}{#2{\upshape(#1)}#3}
\crefname{lp}{LP}{LP}
\creflabelformat{lp}{#2{\upshape(#1)}#3}
\usepackage{aliascnt}
\usepackage{cleveref}
\crefname{ineq}{Inequality}{Inequality}
\creflabelformat{ineq}{#2{\upshape(#1)}#3}
\crefname{sub}{Subsection}{Subsection}
\creflabelformat{Subsection}{#2{\upshape(#1)}#3}
\crefname{sdp}{SDP}{SDP}
\creflabelformat{sdp}{#2{\upshape(#1)}#3}
\crefname{lp}{LP}{LP}
\creflabelformat{lp}{#2{\upshape(#1)}#3}

\newtheorem{theorem}{Theorem}[section]

\newtheorem{lemma}[theorem]{Lemma}
\newtheorem{informal theorem}[theorem]{Theorem (informal statement)}

\newtheorem{proposition}[theorem]{Proposition}

\newtheorem{claim}[theorem]{Claim}
\newtheorem{fact}[theorem]{Fact}

\newtheorem{remark}[theorem]{Remark}

\newtheorem{definition}[theorem]{Definition}
\usepackage[numbers]{natbib}

\newcommand{\eqdef}{\coloneqq}

\newcommand{\lp}{\left}
\newcommand{\rp}{\right}
\newcommand\norm[1]{\left\| #1 \right\|}
\newcommand\snorm[2]{\left\| #2 \right\|_{#1}}
\renewcommand\vec[1]{\mathbf{#1}}
\DeclareMathOperator*{\pr}{\mathbf{Pr}}
\DeclareMathOperator*{\E}{\mathbf{E}}

\newcommand{\normal}{\mathcal{N}}

\newcommand{\bx}{\mathbf{x}}

\newcommand{\bw}{\mathbf{w}}

\newcommand{\R}{\mathbb{R}}

\newcommand{\eps}{\epsilon}

\newcommand{\poly}{\mathrm{poly}}

\newcommand{\sgn}{\mathrm{sign}}
\newcommand{\sign}{\mathrm{sign}}

\newcommand{\OPT}{\mathrm{opt}}

\newcommand{\wt}{\widetilde}

\newcommand{\x}{\vec x}

\newcommand{\opt}{\mathrm{opt}}

\newcommand{\iid}{{i.i.d.}\ }
\newcommand{\abs}[1]{\lp| #1 \rp|}

\renewcommand\Pr{\pr}

\newenvironment{Ualgorithm}[1][htpb]{\def\@algocf@post@ruled{\kern\interspacealgoruled\hrule  height\algoheightrule\kern3pt\relax}\def\@algocf@capt@ruled{under}\setlength\algotitleheightrule{0pt}\SetAlgoCaptionLayout{centerline}\begin{algorithm}[#1]}
{\end{algorithm}}
\DeclarePairedDelimiter\ceil{\lceil}{\rceil}

\def\nnewcolor{0}
\ifnum\nnewcolor=1
\newcommand{\inote}[1]{\footnote{{\bf [[Ilias: {#1}\bf ]] }}}
\newcommand{\dnote}[1]{\footnote{{\bf [[Daniel: {#1}\bf ]] }}}
\newcommand{\nnote}[1]{\footnote{{\bf [[Nikos: {#1}\bf ]] }}}
\newcommand{\snote}[1]{\footnote{\textcolor{orange}{sihan: #1}}}

\newcommand{\new}[1]{\textcolor{red}{#1}}
\newcommand{\snew}[1]{\textcolor{orange}{#1}}
\newcommand{\nnew}[1]{\textcolor{purple}{#1}}

\else
\newcommand{\new}[1]{#1}
\newcommand{\snew}[1]{#1}
\newcommand{\nnew}[1]{#1}

\newcommand{\inote}[1]{}
\newcommand{\dnote}[1]{}
\newcommand{\nnote}[1]{}
\newcommand{\vnote}[1]{}
\newcommand{\snote}[1]{}
\fi

\title{Efficient Testable Learning of General Halfspaces \\ with Adversarial Label Noise\footnote{Presented to COLT'24.}}
\author{
Ilias Diakonikolas\thanks{Supported by NSF Medium Award CCF-2107079 and a DARPA Learning with Less Labels (LwLL) grant.}\\
UW Madison\\
{\tt ilias@cs.wisc.edu}\\
\and
Daniel M. Kane\thanks{Supported by NSF Award CCF-1553288 (CAREER) and NSF Medium Award CCF-2107547.}\\
UC San Diego\\
{\tt dakane@ucsd.edu }\\
\and 
Sihan Liu\thanks{Supported in part by NSF Award CCF-1553288 (CAREER) and NSF Medium Award CCF-2107547.}\\
UC San Diego\\
{\tt sil046@ucsd.edu}
\and
Nikos Zarifis\thanks{Supported in part by 
NSF Medium Award CCF-2107079.}\\
UW Madison\\
{\tt zarifis@wisc.edu}\\
}

\begin{document}

\maketitle

\begin{abstract}%
We study the task of testable learning of general --- 
not necessarily homogeneous --- halfspaces 
with adversarial label noise with respect to the 
Gaussian distribution. In the testable learning framework, 
the goal is to develop a tester-learner such that if the data passes the tester, then one can trust the output of the robust learner on the data.
Our main result is the first polynomial time tester-learner for general halfspaces that achieves dimension-independent misclassification error. 
At the heart of our approach is a new methodology to reduce testable 
learning of general halfspaces to testable learning of \snew{nearly} homogeneous halfspaces that may be of broader interest. 
\end{abstract}

% \begin{keywords}%
% PAC learning, testable learning, adversarial label noise, linear threshold functions
% \end{keywords}

\setcounter{page}{0}
\thispagestyle{empty}
\newpage
\section{Introduction}

A (general) halfspace or Linear Threshold Function (LTF) is any
Boolean function $h: \R^d \to \{ \pm 1\}$ of the form 
$h_{\bw}(\bx) = \sgn \left( \bw \cdot \bx + t \right)$, 
where $\bw \in \R^d$ is the defining vector, $t \in \R$ is the threshold, 
and $\sign: \R \to \{ \pm 1\}$ is defined as $\sgn(t)=1$ if $t \geq 0$ 
and $\sgn(t)=-1$ otherwise.
The family of halfspaces is one of the most basic concept classes in computational
learning theory with  history dating back to Rosenblatt's perceptron algorithm~\citep{Rosenblatt:58}. 
While halfspaces are efficiently learnable
in Valiant's (realizable) distribution-free PAC model~\citep{val84} 
(i.e., with clean labels), in the agnostic (or adversarial label noise) model~\citep{Haussler:92, KSS:94} even {\em weak} learning is intractable~\citep{Daniely16, DKMR22b, Tiegel22}, unless one makes assumptions on the distribution of feature vectors. 

A long line of work, starting with~\citep{KKMS:08}, 
has developed efficient halfspace learners in the distribution-specific
agnostic model. Concretely,~\citep{KKMS:08} gave an agnostic learner
within 0-1 error of $\opt+\eps$ ---  
where $\opt$ is the 0-1 error of the best-fitting halfspace --- 
with complexity $d^{O(1/\eps^2)}$ if the distribution on feature vectors 
is the standard Gaussian. Unfortunately, if one insists on optimal error of $\opt+\eps$,
computational limitations arise even in the Gaussian setting. Specifically, the  
exponential complexity dependence on $1/\eps$ has been recently shown to be 
inherent~\citep{DKZ20, GGK20, DKPZ21, DKR23}.

By relaxing the desired accuracy to $f(\opt) +\eps$, 
for an appropriate function $f(t)$ that goes to $0$ when $t \rightarrow 0$, 
it is possible to obtain $\poly(d/\eps)$ (i.e., {\em fully} polynomial) 
time algorithms~\citep{KLS:09jmlr, ABL17, Daniely15, DKS18a, DKTZ20c, diakonikolas2022learning}.
Specifically, for the class of {\em homogeneous} halfspaces --- corresponding to 
the case that the threshold $t= 0$ or equivalently that the separating hyperplane goes through the origin --- \cite{ABL17} first gave a $\poly(d/\eps)$ time algorithm 
with error $C \opt+\eps$, for some universal constant $C>1$, that succeeds
under the standard Gaussian (and, more generally, isotropic logconcave distributions). 
Interestingly, the case of {\em general} halfspaces turns out to be more challenging:  
with the exception of~\cite{DKS18a, diakonikolas2022learning}, all prior approaches inherently fail for non-homogeneous halspaces. \new{We will return to this point in the subsequent discussion.}

%\new{
%\begin{itemize}
%\item Intro paragraph about halfspaces and distribution-specific agnostic learning as before. 
%\item Definition of testable learning.
%\item No prior work for general halfspaces. Very little even without testable requirement. 
%\item
%\end{itemize}
%}

The model of {\em testable learning}~\citep{RV22a} was defined
to alleviate the following conceptual limitation of 
distribution-specific agnostic learning: 
if the assumptions on the marginal distribution 
on examples (e.g., Gaussianity or log-concavity) are not satisfied, a vanilla
agnostic learner provides no guarantees. 
Ideally, one would like to guarantee the following desiderata:
(a) if the learner accepts, then we can trust its output, 
and (b) it is unlikely that the learner rejects if the data satisfies the distributional assumptions. This has led to the following definition:

\begin{definition}[Testable Learning with Adversarial Label Noise~\citep{RV22a}]
Fix $\epsilon, \tau \in (0,1]$ and let $f:[0,1] \mapsto \R_+$.
A tester-learner $\cal{A}$ (approximately) {\em testably learns} a concept class 
$\mathcal C$ with respect to the distribution $D_\x$ on $\R^d$
 with $N$ samples, and failure probability $\tau$ 
if the following holds. 
For any distribution $D$ on $\R^d\times\{\pm1\}$,
the tester-learner $\cal A$ draws a set $S$ of $N$ i.i.d.\ samples from $D$.
% Then, it can choose to query the labels of $q$ samples from $S$.
In the end, it either rejects $S$ or accepts $S$ and produces a hypothesis $h:\R^d \mapsto \{\pm1\}$. Moreover, the following conditions must be met:
\begin{itemize}[leftmargin=*]
    \item (Completeness)  If $D$ truly has marginal $D_\x$, $\cal{A}$ accepts with probability at least $1-\tau$.
    \item (Soundness) 
    The probability that $\cal A$ accepts and outputs a hypothesis $h$ for which  $\pr_{(\x,y) \sim D}[h(\x) \neq y] > f(\OPT) +\eps$ ,  where 
    $\opt: = \min_{g \in \cal{C}} \pr_{(\x,y) \sim D}[g(\x) \neq y]$ is at most $\tau$.
\end{itemize}
The probability in the above statements is over the randomness of the sample $S$ 
and the internal randomness of the tester-learner $\cal A$.
\end{definition}
Since the introduction of the model, algorithmic aspects of 
testable learning have been investigated in a number of 
works~\citep{RV22a,Gollakota2022, diakonikolas2023efficient, GKSV23, 
gollakota2023tester}. The earlier works~\citep{RV22a, Gollakota2022} studied
the agnostic setting where $f(t) = t$ (corresponding to error $\opt+\eps$). 
These works developed general moment-matching methodology that 
gives testable learners for general halfspaces (and other concept classes). 
Since testable learning is at least as hard as the standard agnostic setting, 
the resulting algorithms necessarily incur 
exponential dependence in $1/\eps$. 
Subsequent work~\citep{diakonikolas2023efficient, GKSV23, gollakota2023tester}
developed testable learners with weaker error guarantees that run in fully-polynomial 
time. Specifically, these algorithms achieve error $O(\opt)+\eps$ under the Gaussian distribution, and more recently, for a subclass of log-concave distributions. {\em Importantly, all of these algorithms work for {\em homogeneous} halfspaces 
and inherently fail for general halfspaces.} 

The distinction between homogeneous and general halfspaces might seem 
inconsequential at first glance. 
Indeed, one can trivially reduce a general halfspace to a homogeneous 
one by adding an extra constant coordinate to each datapoint. 
While this is a valid reduction for distribution-free learning, 
it alters the marginal distribution on the feature vectors. 
Therefore, it does not generally work in the distribution-specific setting 
we study here. \new{Beyond this, as we will elaborate in the proceeding section, the 
techniques underlying previous testable learning algorithms are insufficient to obtain any non-trivial error guarantee for general halfspaces. }

Motivated by this gap in our understanding, 
in this work we ask whether it is possible to obtain a fully-polynomial time testable learner for general halfspaces with dimension-independent error. 
Specifically, we study the following question:
\begin{center}
{\em Is there a $\poly(d/\eps)$ time tester-learner for {\em general} halfspaces with error $f(\opt)+\eps$? \\ }
\end{center}
As our main result, formally stated below, we provide an affirmative answer to this question. 
%\new{[Add more?]}

%\new{[Summarize prior work, mention difficulty of general LTFs, frame open question]}

%\subsection{Our Result}
\begin{theorem} [Testable Learning General Halfspaces under Gaussian Marginals]\label{thm:testable-learning-agnostic}
Let $\eps, \tau \in (0,1)$ and $\mathcal C$ be the class of general halfspaces on $\R^d$.
There exists a tester-learner 
for $\mathcal C$  with respect to $\normal(\vec 0, \vec I)$ up to $0\text{-}1$ error $ \widetilde O \lp( \sqrt{\opt} \rp) + \eps$, where $\OPT$ is the $0\text{-}1$ error of the best fitting function in $\mathcal{C}$, that fails with probability at most $\tau$.
Furthermore, the algorithm draws $N = \poly(d, 1/\eps) \log(1/\tau)$ samples, and runs in time $\poly( N  )$.
\end{theorem}

We reiterate that this is the first polynomial-time testable learner 
for \emph{general} halfspaces that achieves non-trivial \new{(i.e., dimension-independent)} error 
guarantee. As already mentioned, prior works either focus on 
\emph{homogeneous} halfspaces or incur super-polynomial runtime. 

\new{We leave open the questions of obtaining quantitatively better error guarantee (with $O(\opt)+\eps$ as the ideal goal) and extending to broader classes of distributions beyond the Gaussian. Either potential improvement stumbles upon non-trivial obstacles; see \Cref{rem:broader-distr} and \ref{rem:o-opt}}.

\subsection{Some natural attempts and why they fail}
\label{sec:barrier}

One may wonder whether previously known testable learners 
for homogeneous halfspaces can be applied 
\new{(directly or with minor modifications)} in our setting 
to achieve non-trivial learning error, e.g., $\opt^{c}$ for some $c \in (0, 1)$. 
We will demonstrate below that known methods fail to achieve testable learning for 
general halfspaces with even (sufficiently small) constant error.

\vspace{0.1cm}

\noindent {\bf Difficulty of Estimating Chow Parameters}
For learning general halfspaces under the Gaussian distribution 
\new{with adversarial label noise}, 
the only prior works that achieve \new{dimension-independent} error %$O(\opt)+\eps$ 
are \citep{DKS18a} and \citep{diakonikolas2022learning} 
\new{(in the non-testable setting)}. In both of these works, 
the algorithms are required at some point to estimate the 
Chow-parameters, i.e., $\E_{  (\x, y) \sim D}[ y\x ]$.
The intuition behind this choice is that for any halfspace  $h(\x)=\sign(\vec w\cdot\x+t)$,  we have that 
$\E_{  (\x, y) \sim D}[ h(\x)\x ]=\vec w G(t)$,
where $G(t)$ is the pdf of the standard normal. In the non-testable regime, this \new{expectation} is easily computable with $\poly(d/\eps)$ samples and runtime. Unfortunately, in the testable regime, one needs to first certify that the \new{marginal} distribution \new{$D_{\x}$} is similar to a Gaussian in the sense that $\E_{\x \sim \new{D_{\x}}}[ h(\x)\x ]\sim \vec w G(t)$ for any halfspace $h$. 
To the best of our knowledge, the only way to certify this property 
is to certify that at least $\Omega(1/G(t))$ moments 
of the distribution \new{$D_{\x}$} match \new{those} 
of the standard normal (see \cite{Gollakota2022}), 
which requires at least $d^{\Omega(1/G(t))}$ \new{samples and} runtime. In other words, this \new{approach would} lead to exponential runtime 
if $1/G(t)$ is on the order of $\poly(1/\eps)$.

\vspace{0.1cm}

\noindent {\bf Adapting Testable Learners for Homogeneous Halfspaces} 
For testable learning of homogeneous halfspaces, two techniques have been used in prior work. 
The first one uses an adaptive localization method~\citep{diakonikolas2023efficient}, and the second 
is to construct non-convex SGD tester-learners \citep{GKSV23,gollakota2023tester}. At a high-level, 
both of these techniques localize in a band {\em close to the origin} in order to maximize the error conditioned on the band of the current hypothesis and the optimal one. Unfortunately, for general halfspaces, 
one cannot localize to maximize the error, 
{\em unless} the angle between the current hypothesis 
\new{(weight vector)} 
and the optimal one is sufficiently small. 
This holds because if the angle is large, 
then no band can contain large enough error, 
except bands that are very far away (which due to the noise model can be arbitrarily noisy). 
The only \new{known} method to obtain a good enough 
initialization point is to rely on Chow-parameters --- 
but as we discussed above, this is hard in the testable regime.

\subsection{Overview of Techniques}
\label{sec:techniques}

\noindent {\bf Reduction to Nearly Homogeneous Halfspaces} 
% \footnote{Maybe change to Bounded Bias Halfspaces as the threshold initially is $\sqrt{\log(1/\eps)}$ and not arbitrarily small} 
The overall strategy of our algorithmic approach is to \new{efficiently} reduce the testable learning of general halfspaces 
to the testable learning of ``nearly'' homogeneous halfspaces, i.e., halfspaces with \snew{thresholds of size $\eps$} \footnote{Our main technical contribution is to show that we can efficiently construct a list of $\tilde O(1/\eps)$ candidate learning instances one of which is guaranteed to correspond to a halfspace with a threshold of size at most $\eps$; \snew{and reduce testable learning of the unknown general halfspace to testable learning of these nearly homogeneous halfspaces.}}. 
\snew{Assuming such an efficient reduction,  
our main theorem then follows via a careful analysis  
showing that (essentially) 
the main algorithm of the prior work by \citep{diakonikolas2023efficient} (on testable learning of homogeneous
halfspaces) 
can testably learn halfspaces with small biases.}
% testable learning algorithm for homogeneous halfspaces can be used (with a careful analysis showing that they can indeed tolerate a small offset) to testably learn the defining vector of the halfspace. 
% To finish learning the original general halfspace, we construct a list of candidate halfspaces with the learned defining vector and different offsets, and in the end pick the one with the smallest empirical error.
The key \new{notion} in implementing such a reduction is what we refer to as \emph{Good Localization Centers} (\Cref{def:good-localization-center}).
In short, a good localization center is a point that \new{is} 
close to the decision boundary of the \new{target} halfspace 
\new{and} at the same time not too far from the origin. As an 
example, the intersection between the line along the defining 
vector of the target halfspace
% \inote{which halfspace?} 
and the 
separating hyperplane of the halfspace will be a %perfectly 
good localization center.

% -------
Then, \new{leveraging} the technique developed in \cite{DKS18a}, one can localize to regions around 
the good localization center via rejection sampling. 
After such a localization procedure, using the fact that the good localization center is almost on the \snew{target} halfspace, 
one can show that \new{a} general halfspace 
will become nearly homogeneous.
Moreover, since the center is also required to be not too far from the origin, one can further demonstrate that a decent fraction of the samples can still survive the rejection sampling, ensuring that the proceeding invocation of the homogeneous halfspace tester-learner is \new{sample and computationally} efficient. 
\snew{
Our main technical contribution is an efficient testable learner that computes a list of candidates containing at least one good localization center. 
This implies that the learner has inherently learned a good constant approximation to the defining vector of the unknown halfspace, which is the major obstacle against applying prior methods (developed for homogeneous halfspaces) 
in the general halfspace setting 
(see \Cref{sec:barrier} for more details).}

\vspace{0.1cm}

\noindent {\bf Finding a Good Localization Center} 
We now describe a testable learning procedure that either produces a good localization center or reports that the underlying distribution is not Gaussian.
\new{We can assume that the optimal halfspace is 
$h(\x) = \sgn( \vec v^\ast \cdot \x - t^\ast)$ 
for some $t^{\ast}>0$. As mentioned in the preceding discussion,} 
the ideal localization center \new{would be} 
the intersection between the line along the defining vector $\vec v^\ast$ and the separating hyperplane of the \snew{target} halfspace.
% Therefore, ideally, $\vec w$ should be along the direction of $\vec v^\ast$ since this direction gives the shortest distance from the origin to the halfspace.

\nnew{
Naively, one may try to estimate the Chow parameters 
$\vec {\tilde v} := \E_{(\x, y)\sim D}[ y \x ]$, 
and pick some point along the direction of $\vec {\tilde v}$ 
--- as \snew{the angle between $\vec {\tilde v}$ and $\vec v^\ast$ ought to be} small when the underlying $\x$-marginal is Gaussian. 
Yet, as we discussed in \Cref{sec:barrier}, this cannot be done 
efficiently in the testable regime. 
For this reason, we instead calculate the mean of the points that are positively 
labeled (\snew{formally}, the tail points; see \Cref{def:tail-point}), 
i.e., $\bar{\boldsymbol\mu}_+=\E_{(\x,y)\sim D}[\x\mid y=1]$. 
\new{In the absence of corrupted labels},  %no outliers, 
this gives us a point located in the positive side of the halfspace, 
i.e., $h(\bar{\boldsymbol\mu}_+)=1$, 
as $\vec v^\ast \cdot \bar{\boldsymbol\mu}_+=\E_{(\x,y)\sim D}[\vec v^\ast \cdot \x\mid h(\x)=1]\geq t^\ast$. 
\snew{In other words}, there exists a $\lambda\in(0,1)$ so that 
$\lambda \bar{\boldsymbol\mu}_+\cdot\vec v^\ast=t^\ast$, 
and if we center our distribution at the point $\vec w=\lambda \bar{\boldsymbol\mu}_+$ 
\snew{via rejection sampling (see \Cref{def:rejection-sampling})}, 
then the optimal halfspace becomes \snew{exactly} homogeneous (\Cref{fig:mu-def}).}

\nnew{
There are two obstacles that we need to circumvent in order for such an approach to work: (i) We do not have access to the true labels 
and, by conditioning on the positive labels, 
the noise can make the \snew{corrupted} mean 
$\boldsymbol\mu_+=\E_{(\x,y)\sim D}[\x\mid y=1]$ appear on the negative side, i.e., $h(\boldsymbol\mu_+) = -1$; and 
(ii) the point $\lambda \boldsymbol\mu_+$ can be very far from the origin, 
i.e., $\|\lambda \boldsymbol\mu_+\|_2\gtrsim \sqrt{\log(1/\opt)}$, 
and \snew{applying rejection sampling to re-center the distribution at this point} 
can result in a corruption level that is far worse than the current one. 
Fortunately, we can ensure that these \new{situations} cannot happen 
by certifying the following two properties of the empirical distribution 
over samples: 
(i) the distribution has bounded covariance, and
(ii) the cumulative density function of the distribution projected 
along the direction of $\boldsymbol\mu_+$ is sufficiently close 
to that of the standard Gaussian.}

\new{ First, if we assume that (a) the mass of the positive points is at least 
$B\geq \sqrt{\opt} \poly(\log(1/\opt))$, and (b) the sample covariance has 
\snew{spectral norm bounded from above by $2$}, 
\snew{then we can ensure that $\|\boldsymbol\mu_+-\bar{\boldsymbol\mu}_+\|_2\leq 1/\poly(\log(1/\opt))$ via the following simple calculation}:
for any unit vector $\vec u$, $|\E_{(\x,y)\sim D}[\vec u\cdot \x (\mathbbm 1\{h(\x)=1\}- \mathbbm 1\{y=1\})]|\leq \E_{(\x,y)\sim D}[(\mathbbm 1\{h(\x)=1\}-\mathbbm 1\{y=1\})^2]^{1/2}\E_{(\x,y)\sim D}[(\vec u\cdot\x)^2]^{1/2}\leq \sqrt{2\opt}$\snew{; hence,} 
the error $\|\boldsymbol\mu_+-\bar{\boldsymbol\mu}_+\|_2$ is at most 
\snew{$\sqrt{2\opt} / B \leq 1/\poly(\log(1/\opt))$}. Note that 
\snew{assumption (b) can be efficiently verified} and
assumption (a) on $B \geq \sqrt{\opt}\poly(\log(1/\opt))$ is without loss of generality --- otherwise, the constant hypothesis, 
i.e., $h(\x) = -1$ for all $\x$, would get $\wt O(\sqrt{\opt})$ error
\footnote{In the end, we can simply compare the halfspace learned 
via localization with the constant hypothesis in terms 
of their empirical errors to select the best.}. 
This suffices to show that the intersection point $\vec w$ 
is close to $\boldsymbol{\mu}_+$.
\snew{This sketch is formalized in \Cref{lem:mu-v-distance}.}
}

\new{It remains to show that $\|\boldsymbol{\mu}_+\|_2$ is not very large. 
To this end, we certify that the empirical distribution projected 
along $\boldsymbol  \mu_+$ has its cumulative density function 
close to that of the standard Gaussian. 
\snew{If so, $\snorm{2}{\boldsymbol{\mu}_+}$ can then be bounded from below 
by $\Phi^{-1}(M)$, where $M$ denotes the mass of the samples 
inside the region $\{ \x \in S: \boldsymbol  \mu \cdot \x > \snorm{2}{\boldsymbol  \mu}^2 \}$.}
Conditioned on that the ``Gaussian closeness'' CDF test passes,
we show that the distribution of the tail points must be approximately 
``centered'' around $\boldsymbol \mu$ in the $\boldsymbol \mu$-direction.
Thus, $\{ \x \in S:\boldsymbol \mu \cdot \x > \snorm{2}{\boldsymbol \mu}^2 \}$ 
must contain a large fraction of the tail points, allowing us to derive a lower 
bound on the mass of the set $M$, \snew{from which the upper bound on 
$\|\boldsymbol{\mu}_+\|_2$ follows. See \Cref{lem:mean-distance-bound} for the 
details of the argument.}}

\begin{figure}[ht]
    \centering
\includegraphics[width=8cm]{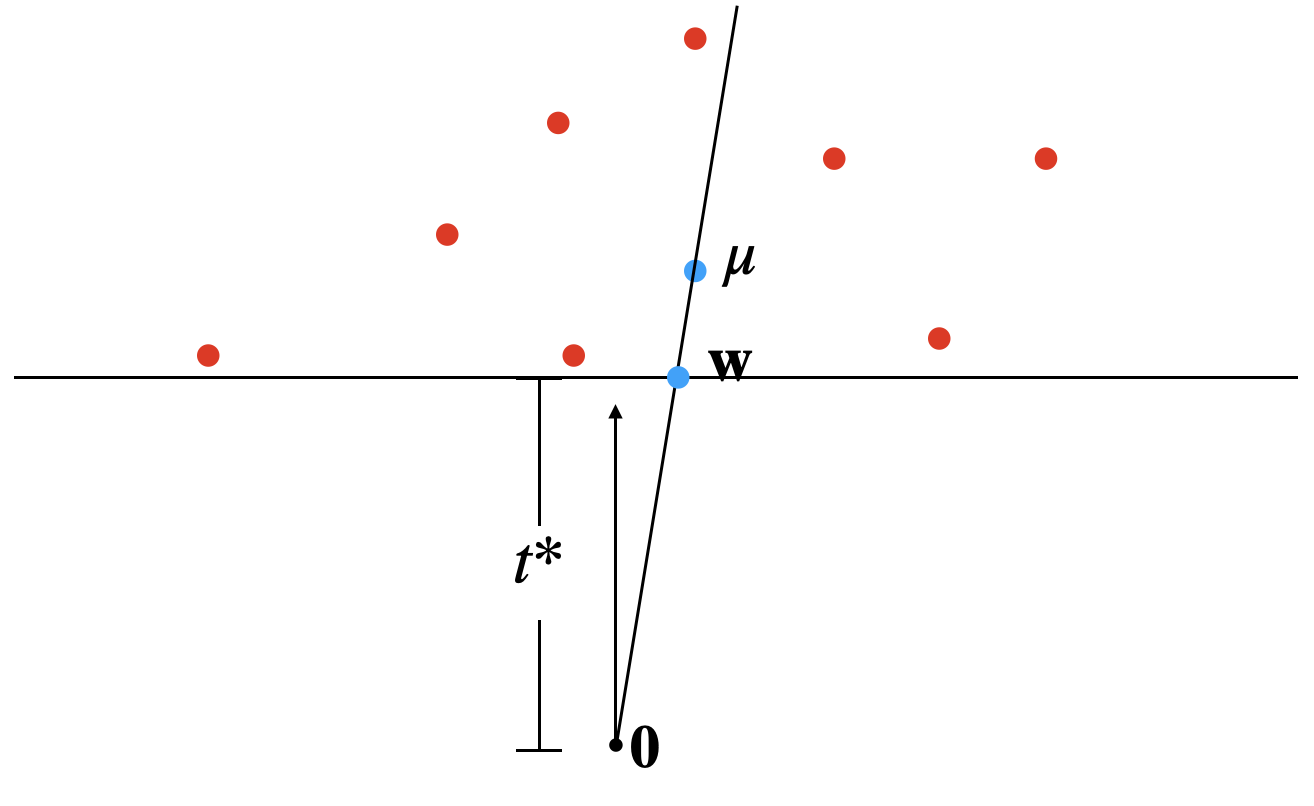}
    \caption{When there is no noise, the tail points are at distance at least $t^\ast$ from the origin, since they all lie on the side of the \new{hyperplane not containing} the origin. Consequently, the line along their mean $\boldsymbol \mu$ must first intersect the separating hyperplane of the halfspace (crossing $\vec w$) and then cross the mean vector $\boldsymbol  \mu$ of the tail points, regardless of the underlying marginal distribution.
    If we re-center the distribution at $\vec w$, the halfspace will then become exactly homogeneous.
    }
    \label{fig:mu-def}
\end{figure}

Conditioned on the relevant tests all passing, 
we can construct a list of candidate localization centers 
along the direction of $\boldsymbol \mu_+$ whose $\ell_2$-norms 
form an $\eps$-cover of the segment $[0, O( \sqrt{ \log(1/B) })]$. \snew{It is not hard to see that such a list is guaranteed 
to contain one  localization center as good as the intersecction point $\vec w$.}
We then perform the reduction to \new{near-}homogeneous halfspaces 
with each of the candidate localization centers in our \new{list}. 
Finally, we simply choose the best halfspace thus obtained, 
by examining the empirical errors of the list of the halfspaces learned.
\snew{This concludes the overview of our approach.}

\subsection{Preliminaries}
We use small boldface characters for vectors and capital bold characters for matrices.
We use $[d]$ to denote the set $\{1, 2, \ldots, d\}$.
For a vector $\vec x \in \R^d$ and $i \in [d]$, $\vec x_i$ denotes the $i$-th coordinate of $\vec x$, and $\|\vec x\|_2 := \sqrt{ \sum_{i=1}^d \vec x_i^2 }$ the $\ell_2$ norm of $\vec x$.
We use $\vec x \cdot \vec y := \sum_{i=1}^n \vec x_i  \vec y_i$ as the inner product between them. We use $\mathbbm 1\{ E  \}$  to denote the indicator function of some event $E$. 

We use $\E_{\vec x \sim D}[\vec x]$ for the expectation of the random variable $\vec x$ according to the distribution $D$ and $\Pr[E]$ for the probability of event $E$. For simplicity of notation, we may omit the distribution
when it is clear from the context. 
For $\boldsymbol \mu \in \R^d, \vec \Sigma \in \R^{d \times d}$, we denote by $\normal (\boldsymbol \mu, \vec \Sigma)$  the $d$-dimensional Gaussian distribution
with mean $\boldsymbol \mu$ and covariance matrix $\vec \Sigma$.
For $(\vec x, y) \in \mathcal X$ distributed according to $D$, we denote $D_{\vec x}$
to be the marginal distribution of $\vec x$.
Let $f: \R^d \mapsto \{\pm 1\}$ be a boolean function and $D$ a distribution over $\R^d$. The (degree-$1$) Chow parameter vector of $f$ with respect to $D$ is defined as $\E_{\vec x \sim D}\lp[f(\vec x) \vec x  \rp]$. 
For a halfspace $h(\vec x) = \sgn(\vec v \cdot \vec x + t)$, we say that $\vec v$ is the defining vector of $h$ \snew{and $t$ is the threshold of $h$}.

We denote by $\Phi(t)$ the cumulative density function (CDF) of the one-dimensional standard Gaussian, i.e., $\Phi(t) = \Pr_{ x \sim \normal(0, 1) }[x > t]$, and by $G(t)$ its probability density function, i.e., $G(t) = \frac{1}{ \sqrt{2\pi} } \exp\lp( - x^2/2 \rp)$.
An elementary inequality we use
%regarding $\Phi(t)$ 
is that  $\Phi(t) \leq \frac{1}{ t \; \sqrt{2 \pi} }\exp(-t^2/2)
= G(t) / t$.
We say that two one-dimensional distributions $X,Y$ are $\eps$-close in CDF (or Kolmogorov) distance if they satisfy the inequality~$\sup_{t \in R} \abs{\Pr[ X > t ] - \Pr[Y > t]} \leq \eps$.

The asymptotic notation $\tilde{O}$ (resp. $\tilde{\Omega}$) suppresses logarithmic factors in its argument, 
i.e., $\tilde{O}(f(n)) = O(f(n)\log^c f(n))$ and 
$\tilde{\Omega}(f(n)) = \Omega(f(n)/\log^c f(n))$, where $c>0$ is a universal constant. 
We write $a \lesssim b$ (resp. $a\gtrsim b$) to denote that $\alpha \leq c b$ for a sufficiently small  absolute constant $c>0$ (resp. for a sufficiently large  absolute constant $c>0$).

% Formally, we write $f(n) \ll g(n)$ to denote that $f(n) < c \cdot g(n)$, for some universal constant.

% \paragraph{Moment-Matching}  In what follows, we use the phrase 
% \emph{``A distribution $D$ on $\R^d$ matches $k$ moments with a distribution $Q$ up to error $\Delta$''}.  Similarly to \cite{Gollakota2022}, we formally define approximate moment-matching as follows.
% \begin{definition}[Approximate Moment-Matching]
% Let $k\in \mathbb N$ be a degree parameter and let $\mathcal{M}(k, d)$ be the set of $d$-variate monomials of degree up to $k$.  Moreover, let 
% $\vec \Delta \in \R_+^{|\mathcal{M}(k,d)|}$ be a slack parameter (indexed by the monomials of $\mathcal{M}(k,d)$), satisfying
% $\vec \Delta_0 = 0$.  We say that two distributions $D, Q$ match $k$ moments up
% to error $\vec \Delta$ if 
% $|\E_{\x \sim D}[m(\x)] - \E_{\x \sim Q}[m(\x)]| \leq \vec \Delta_m$ for every monomial
% $m(\x) \in \mathcal{M}(k,d)$.
% When the error bound $\Delta$ is the same for all monomials we overload notation and simply use $\Delta$ instead of the parameter $\vec \Delta$.
% \end{definition}

\section{Testable Learning of General Halfspaces to Error \texorpdfstring{$\wt O\lp( \sqrt{\opt} \rp)$}{Lg} }\label{sec:general-halfspaces}
In this section, we establish our main result (\Cref{thm:testable-learning-agnostic}).
%--- an efficient testable learner for general halfspaces that achieves an error of $\wt O( \sqrt{\opt} )$, where $\opt$ is the error achieved by the optimal halfspace.
Our main strategy is to efficiently reduce testable learning of general halfspaces to testable learning of {\em nearly} homogeneous halfspaces, i.e., halfspaces whose \new{thresholds have small magnitude}.
%are of size $\eps$.
% We present a sketch of the analysis of why prior tester-leaner for homogeneous halfspace work for such nearly homogeneous halfspaces as well in Section~\ref{sec:homoge}.
The main \new{technical tool enabling} such a reduction is \new{the notion of} a \emph{good localization center}, defined below. 

\begin{definition}[Good Localization Center]
\label{def:good-localization-center}
Let $\alpha > 0$ and $\beta \in (0,1)$.
Given a general halfspace $h(\x)$, we say that a point $\vec w \in \R^d$ is 
an {\em $(\alpha, \beta)$-good localization center with respect to $h$}, if 
(i) the distance from \snew{the separating hyperplane of} $h$ to $\vec w$ is bounded by $\alpha$, 
and (ii) $\Phi(\snorm{2}{\vec w}) \geq \beta$, where $\Phi$ is the cdf of $\mathcal{N}(0, 1)$.
\end{definition}
\begin{figure}[ht]
    \centering
    \includegraphics[width=8cm]{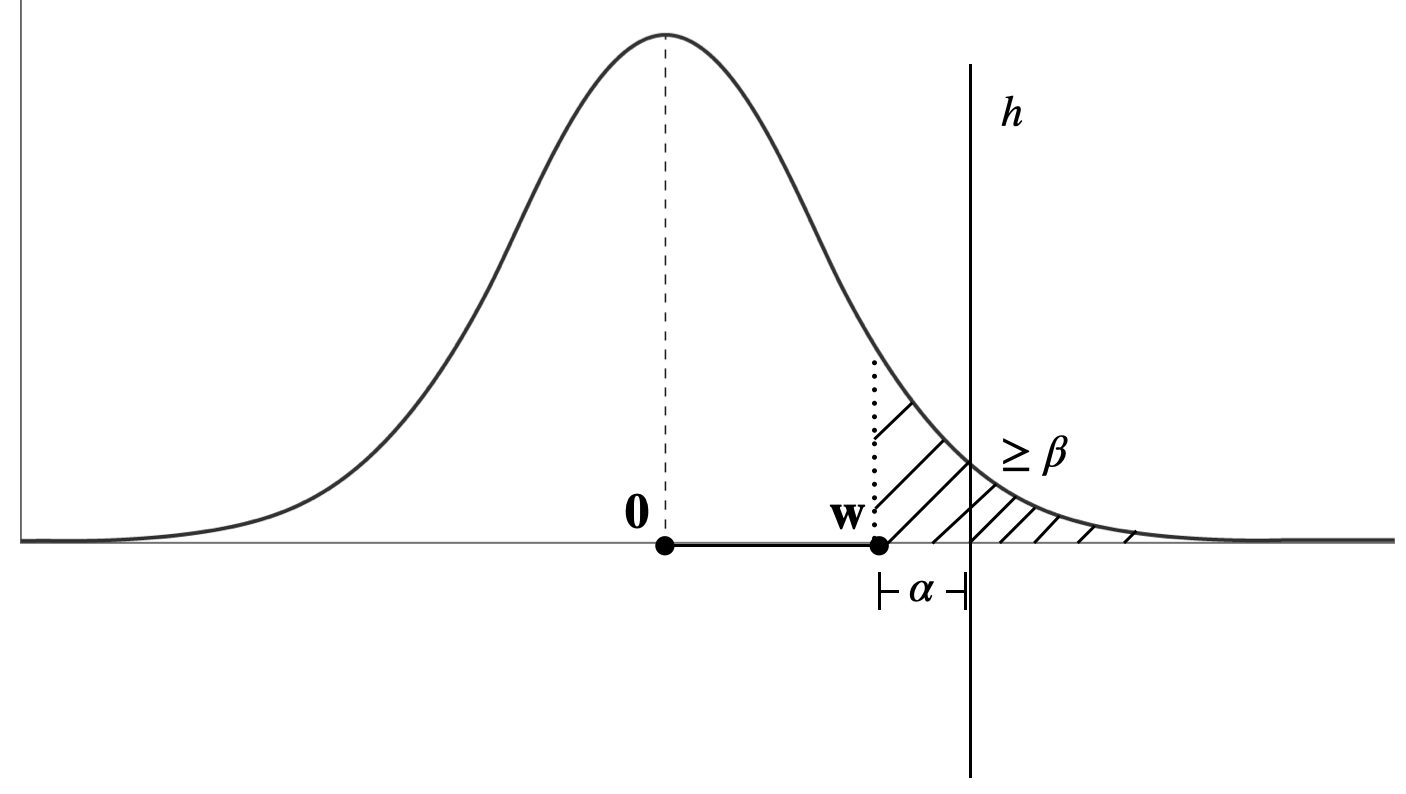}
    \caption{The figure illustrates a good localization center $\vec w$ with respect to the halfspace $h$. $\vec w$ is $\alpha$-far from the halfspace, and $\Phi(\snorm{2}{\vec w})$ is still non-trivial (bounded from below by the mass on one side of the halfspace).}
    \label{fig:localization-center}
    % \vspace{-2em}
\end{figure}
% \inote{Def above is not clear. What is the distance between a function and a point? You likely mean sth else.}
% \inote{Picture?}

In Section~\ref{sec:good-localization-def}, we demonstrate how a good localization center can be \new{leveraged} to perform such a reduction. In the process, \new{we also develop additional required technical machinery}, 
including soft-localization (\Cref{lem:good-center}) 
and a Wedge-Bound (\Cref{lem:general-wedge-bound}).
% In subsection~\ref{sec:good-localization-def}, we present the definition of a good localization center, and explain why such a center can be used to reduce testable learning of general halfspaces into testable learning of nearly homogeneous halfspaces.
In Section~\ref{sec:find-center}, we describe an algorithm 
(\Cref{alg1:localization-center}) that testably constructs a list of 
candidate points containing at least one good localization center; 
this is the main technical contribution of this work.
In Section~\ref{sec:put-together}, we put everything together to 
complete the proof of \Cref{thm:testable-learning-agnostic}.

\subsection{Good Localization Center and its Properties}
\label{sec:good-localization-def}
% We first give the definition of an $(\alpha, \beta)$-\emph{Good Localization Center} with respect to a halfspace.
% % We use $\Phi(\cdot)$ to denote the Gaussian cumulative density function (CDF), i.e. 
% % $\Phi(x) = \Pr_{ x \sim \normal(0,1) }\lp[x > t\rp]$.
% \begin{definition}[Good Localization Center]
% \label{def:good-localization-center}
% Let $\alpha > 0$ and $\beta \in (0,1)$.
% Given a general halfspace $h(\x)$, we say $\vec w$ is an $(\alpha, \beta)$-good localization center with respect to $h$ if (i) the distance from $h$ to $\vec w$ is bounded by $\alpha$, and (ii) $\Phi(\snorm{2}{\vec w}) \geq \beta$.
% \end{definition}
Let $\vec w$ be an $(\alpha, \beta)$-good localization center with respect to the halfspace $h$. 
If $\vec w$ is sufficiently close to the \new{separating hyperplane} of the halfspace and $\beta$ is bounded from below, we show that $\vec w$ can be leveraged to reduce testable learning of general halfspaces to testable learning halfspaces that are \emph{nearly homogeneous}.
Specifically, Lemma~\ref{lem:good-center} demonstrates that if we ``localize'' the $\x$-marginal around the center $\vec w$ using rejection sampling (see Definition~\ref{def:rejection-sampling}), the halfspace will be transformed into a new one whose threshold is bounded above by $O(\alpha \sqrt{\log(1/\beta)} )$.
\begin{definition}\label{def:rejection-sampling}
For $\vec w \in \R^d$ and $\sigma \in (0, 1)$, 
the $(\vec w, \sigma)$-rejection
procedure \new{is as follows:} given $\x \in \R^d$, the procedure accepts it with probability
$
\exp\lp(-(\sigma^{-2}-1)(\vec w/\snorm{2}{\vec w} \cdot \x + 
\snorm{2}{\vec w}/(1 - \sigma^2))^2
/2\rp)  \, ,
$
and rejects it otherwise.    
\end{definition}
% \inote{Intuition?}
\snew{
At a high level, if one performs $(\vec w, \sigma)$-rejection sampling on a Gaussian distribution, the resulting distribution will be $G:=\normal(\vec w, \vec \Sigma)$, where $\vec \Sigma$ is the matrix with eigenvalue $\sigma^2$
in the $\vec w$-direction, and eigenvalues $1$ in all orthogonal directions. 
This follows from the fact that the formula of the acceptance probability in \Cref{def:rejection-sampling} is precisely the ratio between the probability density functions of $G$ and the standard Gaussian.
Let $S_h$ be the separating hyperplane of $h$.
Assume that $\vec w$ is $\alpha$-close to $S_h$. 
This immediately implies that $S_h$ will be at most $\alpha$-far from the center of the new distribution $G$.
If we transform the space to make $G$ isotropic again, $h$ will then consequently get transformed into a new halfspace with small threshold.
This is formally shown in the following lemma, whose proof can be found in Appendix~\ref{app:21}}. 
% Applying the above rejection sampling procedure on the standard Gaussian distribution gives rise to another transformed Gaussian.
% \begin{lemma}
% If elements    
% \end{lemma}
\begin{lemma}[Localization With A Good Center]
\label{lem:good-center}
Let $\vec w$ be an $(\alpha, \beta)$-localization center with respect to 
$h(\x) = \sgn(\vec v^\ast \cdot \x + t^\ast)$ and 
$\sigma \eqdef \min (\|w\|_2^{-1}, \sqrt{1/2})$. The following hold:
\begin{enumerate}[leftmargin=*]
    \item The $(\vec w, \sigma)$-rejection sampling procedure 
    (\Cref{def:rejection-sampling}) applied on the standard Gaussian, 
    results in a distribution $G := \normal(\vec w, \vec \Sigma  )$ 
    with
$\vec \Sigma = \vec I - (1 - \sigma^2) \vec w \vec w^T / \snorm{2}{\vec w}^2$. Moreover, the acceptance probability is
$\Omega(\beta)$, and $\sigma \geq \Omega  (1/\sqrt{\log(1/\beta)}).$
    \item 
    Applying the transformation 
    $(\x, y) \mapsto  (  \vec \Sigma^{-1/2} \lp(\x - \vec w\rp), y)$ to the distribution of $(\x, h(\x))$, where $\x \sim G$, leads to
    the distribution $(\x, h'(\x))$, where $\x \sim \normal(\vec 0, \vec I)$, and 
    $
    h'(\x) = 
\sgn( \vec \Sigma^{1/2} \vec v^\ast \cdot \x + \vec v^\ast \cdot \vec w + t^\ast).
    $
    Moreover, the resulting halfspace $h'$ has its threshold size bounded 
    from above by 
    $ |\vec v^\ast \cdot \vec w + t^\ast| /\|\vec \Sigma^{1/2} \vec v^\ast\|_2
 = O(\alpha \sqrt{\log(1/\beta)}$.
\end{enumerate}
\end{lemma}

Given a good localization center, \Cref{lem:good-center} allows us to transform the original halfspace into one with a small \new{threshold}. 
We show that such a \new{nearly} homogeneous halfspace can be testably learned effectively in the proposition below.
\begin{proposition}[Testable Learning \new{of Nearly} Homogeneous Halfspaces]
\label{prop:small-offset-halfspace}
Let $\eps, \tau \in (0,1)$.
Let $D$ be a distribution over $\R^d \times \{\pm 1\}$ and $h$ be a halfspace that achieves error $\opt$ under $D$ with $h(\x) = \sgn(\vec v^\ast\cdot \x + t^\ast)$, where $\vec v^\ast \in \R^d$ is some unit vector, and $\abs{t} < \eps$.
Then, given $N = \poly(d/\eps) \log(1/\tau)$ many \iid samples from $D$, there exists an efficient tester-learner that runs in time $\poly(N)$, and 
 either reports $D_\x$ is not Gaussian or outputs a vector $\vec v$.
 Moreover, with probability at least $1 - \tau$, we have (i) if the algorithm reports $D_\x$ is not Gaussian, the report is correct, and (ii) if the algorithm returns a vector $\vec v$, $\vec v$ 
 satisfies that $\snorm{2}{\vec v - \vec v^\ast} \leq O(\opt) + \eps$.
\end{proposition}
% \snew{One may hope the guarantees will follow from the testable learner's guarantees on the homogeneous halfspaces in a black-box manner.}
This proposition is established by a careful analysis of the testable learner for homogeneous halfspaces by \cite{diakonikolas2023efficient}.
The main idea is that whenever the threshold of the optimal halfspace is bounded by some absolute constant, the Chow parameters will be a constant multiplicative factor of the defining vector. 
We can testably learn the Chow parameters up to constant error, which then gives a constant approximation to the defining vector of the halfspace. 
After that, the algorithm proceeds in an iterative manner. 
In particular, it performs localization via rejection 
sampling using the current estimate of the defining vector. 
The localization steps may blow up the threshold $t^\ast$ 
of the optimal halfspace by at most a factor of $1/\eps$. 
Since $t^\ast$ is initially bounded from above by $\eps$, 
the threshold stays bounded from above by some absolute 
constant. \snew{This ensures that the Chow-parameters 
vector stays within a constant multiplicative factor of the 
defining vector}, thereby allowing us to iteratively 
approximate the defining vector by testably learning the 
Chow parameters \new{to constant error}.
% use \Cref{prop:small-offset-halfspace} to testably learn the transformed defining vector up to small $\ell_2$ error.
The detailed proof can be found in Appendix~\ref{app:secgeneral}.
% \inote{I would add a bit more here for technical details/difficulties.}

% \inote{Is everything below here for appendix?}
Since \Cref{lem:good-center} gives the form of the transformation on the defining vector exactly, 
after learning the defining vector of the nearly homogeneous halfspace, 
we can revert the transformation to get back the defining vector of the original halfspace. 
While the argument for bounding the error of the vector after reverting the transformation is a standard piece in most localization based learning procedures (see, e.g., \cite{DKKTZ21}),
% In \Cref{lem:transformed-angle}, we show learning the defining vector of the transformed halfspace up to good accuracy is enough for us to retrieve a good approximation of the defining vector of the original halfspace.
% The following lemma bounds the error of the defining vector after we revert the transformation.
% states that if we can learn the defining vector of a halfspace under the distribution localized to a good localization center, we can revert the transformation to obtain a vector closed to the original defining vector. 
% As we will see later, 
one technical issue in our scenario is that such an argument usually requires the angle between the localization direction $\vec w$ and the defining vector $\vec v^\ast$ to be bounded away from $\pi/2$ by some constant. On the other hand, we may need to apply 
the lemma in cases where $\vec w$ and $\vec v^\ast$ are almost orthogonal. 
To deal with this issue, we require the following lemma 
which carefully bounds how the error grows as the angle 
between $\vec w$ and $\vec v^\ast$ increases (see 
Appendix~\ref{app:21} for the proof).
\begin{lemma}[Transformation Error] \label{lem:transformed-angle}
Let $\sigma, \delta \in (0, 1)$, $\beta \in (0, \sqrt{2})$, $\vec w, \vec v, \vec v^\ast \in \R^d$ be unit vectors , and
$\vec \Sigma = \vec I - (1 - \sigma^2) \vec w \vec w^T$.
Furthermore, assume that 
$
\|\vec v^\ast - \vec w\|_2 < \beta,
\| \vec v - \vec \Sigma^{1/2} \vec v^\ast / \|\vec \Sigma^{1/2} \vec v^\ast\|_2  \|_2 <\delta.
$
Then it holds
$
\|
\vec \Sigma^{1/2}
\vec v / \snorm{2}{\vec \Sigma^{1/2}
\vec v} -
\vec v^\ast \|
\leq O ( \delta ) (\sigma + \beta) \;
\big(  \sigma^{-1} \frac{\beta}{1 - \beta^2/2} + 1  \big).
$
\end{lemma}

After learning the defining vector, we search for the 
threshold of the halfspace in a brute-force manner.
In particular, we construct a list of candidate halfspaces 
with $\poly(1/\eps)$ many possible thresholds, and then 
select the one with the smallest empirical error after 
drawing sufficiently many samples.
The list is guaranteed to contain some halfspace $\tilde h$ 
whose defining vector and threshold are both close to those 
of the optimal halfspace, and such a search procedure based 
on the empirical error is guaranteed to yield some 
halfspace whose overall error is no worse than the best 
among the list (up to a constant factor).
The last remaining piece is a procedure to certify that 
closeness to the optimal halfspace in parameter distance 
implies small $0$-$1$ error.
This is achieved by running a Wedge-Bound algorithm, \new{very similar to the one} from \cite{diakonikolas2023efficient}. 
% \inote{is it identical?}
% Lastly, we need a routine to certify that learning the defining vector up to good accuracy implies that 
% Finally, we will need a Wedge Bound Lemma for General halfspaces.
% Essentially, given distribution $D$ and a unit vector $\vec v$, we will run the same Wedge Bound Test as the old algorithm.

Though the algorithm is \new{essentially the same} 
to the one in \cite{diakonikolas2023efficient}, 
we need to generalize
its analysis to handle general halfspaces. 
The formal guarantee is specified in the following lemma
(see \Cref{alg:wedge-bound} for the detailed pseudocode 
and Appendix~\ref{app:21} for the proof).
\begin{lemma}[Wedge Bound for General Halfspaces] \label{lem:general-wedge-bound}
Let $D$ be a distribution over $\R^d \times \{\pm 1\}$.
Let $\vec v, \vec v^\ast$ be two unit vectors in $\R^d$ satisfying
$\snorm{2}{ \vec v^\ast - \vec v } < \delta$ and 
$t, t^\ast$ be two positive numbers in $\R^d$
 satisfying $  0< t - t^\ast <\eta$.
Assume that $D$ passes the tests in \Cref{alg:wedge-bound} with tolerance $\eta$ along the direction $\vec v$.
Denote $h(\x) = \sgn( \vec v \cdot \x + t )$ and $h^\ast(\x) = \sgn(\vec v^\ast \cdot \x + t^\ast)$.
Then it holds
$
\Pr_{\x \sim D_\x}[   h(\x) \neq h^\ast(\x)  ]
\leq O( \delta + \eta ).
$
\end{lemma}
% In \Cref{lem:general-wedge-bound}, we show a variant of the wedge bound lemma from [CITE] for general halfspaces that allows us to certify that small $\ell_2$ error in terms of the defining vector implies small $0-1$ error for the learned halfspace.

%Before proving \Cref{lem:general-wedge-bound}, 
We remark that
though the lemma above suffices for the purpose of learning general halfspaces up to error $\wt O\lp( \sqrt{\opt} \rp)$, 
it is not as efficient as its counterpart for homogeneous halfspaces. 
In particular, for a general halfspace with threshold $t$, if the distribution is indeed Gaussian one should expect some $\Phi(t)$ dependence in the equation, but there is no such dependency in the lemma above.
\new{This is one of the main bottlenecks of our approach 
in achieving error $O(\opt)$ instead of $\wt O(\sqrt{\opt})$.}
% We leave it as an open question whether improvements can be made.

% An important piece in the proof of \Cref{lem:general-wedge-bound} is the following claim that bounds the mass of the disagreement region between two halfspaces restricted to a cylindrical slab. [TODO: add a short proof maybe in appendix.]

\subsection{Finding a Good Localization Center}
\label{sec:find-center}
In this section, we present an efficient algorithm for finding a good 
localization center given that the distribution satisfies certain 
certifiable properties. Our main \new{algorithmic ingredient} returns 
a \new{small} list of points, 
at least one of which is a good localization center.

% \inote{Fix phrasing in statement below}

\begin{proposition} \label{lem:certify-localization-center}
Let $\eps \in (0, 1)$, $D$ be a distribution over $\R^d \times \{\pm 1\}$, and $h(\x) = \sgn(\vec v^\ast\cdot \x + t^\ast)$ be a halfspace with 0-1 error at most $\opt$ with respect to $D$. 
Define the parameter $B \eqdef \min \big( \Pr_{(\x, y) \sim D}[y = +1], \Pr_{(\x, y) \sim D}[y = -1] \big).$
Then 
\Cref{alg1:localization-center} draws $\poly(d/\eps)$ i.i.d.\ samples from $D$, and either reports that the distribution is not Gaussian or returns a list $L$ containing at most $2\log(1/B)/\eps^2$ many points.
Moreover, the following hold with high constant probability:
\begin{enumerate}[leftmargin=*]
    \item If the algorithm reports anything, 
    the report is correct.
    \item If we have that $ \max(\sqrt{\eps}, \sqrt{\opt} \log^2(1/\opt)) < B$, then there exists $\vec w\in L$ such that $\vec w$ is 
    an~$( \eps^2,   C \; B / \log(1/B) )$-\new{good} localization center, where $C>0$ is a universal constant.
\end{enumerate}
% it holds
% $$
% \Pr_{ (\x, y) \sim D }\lp[ \x \cdot  \boldsymbol{\mu}_+ / \snorm{2}{\boldsymbol{\mu}_+}
% > \snorm{2}{ \boldsymbol{\mu}_+} -10\rp]
% \geq  C \; B / \log(1/B) - \eps \, ,
% $$
% where $C$ is a small enough constant.
\end{proposition}
Note that in \Cref{lem:certify-localization-center} we assume that the mass of the points with the ``minority'' label is at least $B \gtrsim \max\lp( \sqrt{\opt} \log^2(1/\opt), \sqrt{\eps} \rp)$ in the completeness case, i.e., the case when the algorithm does not reject. 
\snew{This would be an issue (and thereby a bottleneck) if the goal were to achieve a learning error of $O(\opt)$.
However, we claim that this is a minor assumption when our target error is $\wt O(\sqrt{\opt})$.
Indeed, if the assumption on $B$ does not hold, then there exists a constant halfspace 
($h(\x) \equiv 1$ or $h(\x) \equiv -1$) 
so that $\pr[h(\x)\neq y]\leq B=\wt O(\sqrt\opt)$. Since our algorithm in the end returns the halfspace with the minimum error over testing samples among all the candidate halfspaces found, we can easily meet the error guarantee if we include the constant halfspaces in our hypothesis list.}
% This is fine for the purpose of semi-agnostic learning since in the end, we can always compare the halfspace learned with the trivial halfspace that outputs $+1$ or $-1$ deterministically.
% If we were in the case that the mass of the minority points are bounded from above by $O(1) \; \max\lp( \sqrt{\opt} \log^2(1/\opt), \sqrt{\eps} \rp)$, the trivial halfspace would have already met our learning requirement.

As discussed in \Cref{sec:techniques}, when the halfspace has a non-trivial threshold,
in order to obtain directional information about its defining vector, we need to focus on the \emph{tail points}.
\begin{definition}[Tail point] \label{def:tail-point}
Let $h(\x) = \sgn( \vec v^\ast \cdot \vec x + t^\ast )$ be a halfspace.
% Suppose that $h(\vec 0) = +1$.
If $t^\ast>0$ (resp. $t^\ast<0$) we say that all the points with label $-1$ (resp. $+1$) after corruption are the tail points. 
% Suppose that $h(\vec 0) = -1$. Then we say all the positive points are the tail points.
\end{definition}
Note that we can assume without loss of generality 
that the identity of the tail points is known to us. This is because we can run the same procedure twice --- the first time treating the points with label $+1$ as the tail points and the second time treating the points with label $-1$ as the tail points, and in the end combine the candidate localization centers obtained. 

The algorithm starts by computing the mean vector of the tail points, which we denote by $\boldsymbol \mu$. 
Then, the algorithm (1) tests that the first and second moments of the empirical distribution
are close to those of $\normal(\vec 0, \vec I)$, and (2) tests that the marginal distribution projected along the direction of the mean vector $\boldsymbol{\mu}$ is sufficiently close to the Gaussian distribution. 
Conditioned on the event that the tests pass, the algorithm outputs a list of points that cover the segment from $\vec 0$ to $O(\sqrt{\log(1/B)}) \; \boldsymbol \mu$.
The detailed pseudocode is given in Algorithm~\ref{alg1:localization-center}.

% We first argue the tests performed by \Cref{alg1:localization-center} are sound: when the underlying distribution over the $\x$-marginals is standard normal, the algorithm (incorrectly) rejects with a small probability.

% We begin by discussing the assumptions made in the above proposition.
% Firstly, 
% We have $t^\ast > 10$.
% We begin with some preliminary simplification.
% Throughout the analysis, without loss of generality, we assume the $+1$ points are the ``tail points'' (points that are all far away from the origin in the noise-free case). 
% In the actual algorithm, we try executing all the steps twice while assuming $+1$ 
% If $-1$ are the ``tail points'', we could simply 
% Moreover, we assume that the mass of the $+1$ points is at least $\max(\sqrt{\opt}, \sqrt{\eps})$.
We now proceed to show the list of points returned by the algorithm with high constant probability contains a good localization center.
Specifically, we argue that the intersection between the line along the mean vector $\boldsymbol \mu$ and the separating hyperplane of the optimal halfspace $h$ \footnote{We say that $\x$ is the point of intersection of the separating hyperplane of $h$ and the line along $\vec u$ if $\vec v \cdot \x + t = 0$ and
$\x = \lambda \vec u$ for some $\lambda \in \R$.
% For convenience, we often write it as the intersection of $h$ and the line along $\vec u$.
When the halfspace has non-zero threshold, there will be exactly $1$ intersection point when the line is not orthogonal to the defining vector of the halfspace, and $0$ intersection point otherwise.
}
is a
$\lp(0, \Omega(B / \log(1/B)) \rp)$-good localization center, where $B$ is the mass of the points with the minority label. Since the list is guaranteed to contain some point close to the intersection, it follows that it contains at least one $(\eps, \Omega(B/\log(1/B)))$-good localization center.
% \begin{definition}
% Let $h(\x) = \sgn(\vec v \cdot \x + t)$ be a halfspace, and $\vec u \in \R^d$.
% \end{definition}
% \begin{remark}
% When the halfspace has non-zero threshold, there will be exactly $1$ intersection point when the line is not orthogonal to the defining vector of the halfspace, and $0$ intersection point otherwise.
% \end{remark}
% For now, we also assume that the threshold of the optimal halfspace is at least some absolute constant, which is the more challenging case. 
To show that the intersection point is a good localization 
center, we first argue it cannot be much further from the 
origin than the mean vector $\boldsymbol \mu$.
Second, we argue that the mean vector $\boldsymbol \mu$ itself cannot be too far from the origin.
We now give the formal statement of the first property.
\begin{lemma}[Distance Between Intersection And Mean]
\label{lem:mu-v-distance}
Let $B \in (0, 1)$, $D$ be a distribution over $\R^d \times \{\pm 1\}$, and $h(\x) = \sgn(\vec v^\ast \cdot\x + t^\ast)$ be a halfspace that achieves the optimal error $\opt \in (0, 1)$ with respect to $D$.
Moreover, assume that $|t^\ast| > 10$
and $B \gtrsim \sqrt{\opt} \log^2(1/\opt) + \epsilon$.
Let $S$ be a set of $N=\poly(d/\epsilon)$ \iid samples drawn from $D$, $\boldsymbol \mu$ be the mean vector of the tail points among the samples with respect to $h$ (see \Cref{def:tail-point}),
and $\vec v$ be the point of intersection between the separating hyperplane of $h$ and the line along $\boldsymbol{\mu}$.
Assume that the empirical distribution over $S$ has its covariance matrix bounded from above by $2 \vec I$, and the fraction of tail points among the samples is at least $B$. 
Then with high probability it holds that
$
\snorm{2}{\vec v} \leq \snorm{2}{\boldsymbol{\mu}} + O( 1 / \sqrt{\log(1/B)} ).
$
\end{lemma}
\snew{We now give some high-level ideas of the proof. 
When there are no corruptions, 
we have that all the tail points are on a different 
side of the \new{hyperplane} than the origin.
As illustrated in \Cref{fig:mu-def}, it is easy to see that the line from the origin along $\boldsymbol{\mu}$ must first intersect with the separating hyperplane of the halfspace, and then pass through the mean vector. This immediately gives us $\snorm{2}{\vec v} \leq \snorm{2}{\boldsymbol{\mu}}$, where $\vec v$ is the intersection point.
To deal with the noise, we will use the fact that \Cref{alg1:localization-center} certifies that (i) the mass of the outliers is only a small fraction of the mass of the tail points, and (ii) the sample covariance matrix has 
spectral norm bounded from above by some constant.
The above turns out to be sufficient to certify that the outliers cannot move the mean by more than $O(1 / \log(1/B))$ in $\ell_2$ distance.
\Cref{lem:mu-v-distance} then follows via a careful geometric argument. The detailed proof can be found in Appendix~\ref{app:22}.}

Next we argue that the distance from the mean vector $\boldsymbol\mu$ of the tail point to the origin can be bounded from above such that
$ \Phi(\snorm{2}{\boldsymbol{\mu}}) \geq \Omega( B / \log(1/B) ) $
conditioned on that the tests in \Cref{alg1:localization-center} regarding the empirical distribution projected along the direction of $\boldsymbol\mu$ pass. 
Intuitively, this is due to the fact that a decent fraction 
of the tail points have to lie on the further side of $\boldsymbol\mu$, i.e., the set $\{\vec x: \vec x \cdot \boldsymbol\mu \geq \snorm{2}{\boldsymbol\mu}^2\}$.
\snew{If this were not the case, in order to balance the contributions from the points on the other side of the mean, the tail points would have to locate on the extreme end of the Gaussian tail, which would result in a violation to the CDF test along the direction of $\boldsymbol\mu$.
The formal proof can be found in Appendix~\ref{app:22}.
}
\begin{lemma}
\label{lem:mean-distance-bound}
Let $D$ be a distribution over $\R^d \times \{\pm 1\}$, and $h(\x) = \sgn(\vec v^\ast \cdot\x + t^\ast)$ be a halfspace that achieves the optimal error $\opt$ with respect to $D$, where $|t^\ast| > 10$.
Let $S$ be a set of $N=\poly(d)/\eps^2$ \iid samples from $D$, $\boldsymbol\mu$ be \new{the} mean vector of the tail points among the samples with respect to $h$ (see \Cref{def:tail-point}).
Suppose Lines~\ref{line:cdf-check}, and~\ref{line:stability-check} from \Cref{alg1:localization-center} pass.
% Assume that $S$ has its second moments bounded by $2 \vec I$, and the fraction of tail points among the samples is at least $B$. 
Then it holds that
$
\Phi( \snorm{2}{\boldsymbol{\mu}} ) \geq \Omega(B / \log(1/B)).
$
\end{lemma}

One may have noticed that \Cref{lem:mu-v-distance} and \Cref{lem:mean-distance-bound} require $t^\ast > 10$.
When the assumption does not hold, we claim that there is a 
straightforward way to find a good localization center. In 
particular, we can now just testably learn the Chow 
parameters of the halfspace.
Since now the Chow parameters are of size 
$G(t) = \Theta(1)$, it suffices \new{that} 
we can testably learn it 
up to constant accuracy, which can be readily achieved by 
Lemma 2.3 from \cite{diakonikolas2023efficient}.  
The formal statement and its proof can be found in \Cref{lem:small-offset-center-search} in Appendix~\ref{app:22}.

We are now ready to conclude the proof of \Cref{lem:certify-localization-center}.
\begin{proof}[Proof of \Cref{lem:certify-localization-center}]
We defer the proof of the soundness of the algorithm to \Cref{lem:test-soundness,lem:small-offset-center-search}.
We now argue that the list returned by the algorithm will contain a good localization center with high constant probability.
First, assume that $t^\ast \leq 10$.
The existence of a good localization center follows from \Cref{lem:small-offset-center-search}.
Next we consider the case $t^\ast > 10$.
Let $\boldsymbol \mu$ be the mean vector of the tail points, and $\vec v$ be the intersection of the line along $\boldsymbol \mu$ and the 
\snew{separating hyperplane} of the optimal halfspace $h$.
% From \Cref{lem:mu-v-distance} and \Cref{lem:mu-v-distance}\nnote{fix?}, we have that \nnote{??}.
The argument of $\Phi \lp( \snorm{2}{\vec v} \rp) \geq \Omega(B / \log(1/B))$ relies on the following two claims, which follow from \Cref{lem:mean-distance-bound} and \Cref{lem:mu-v-distance} respectively: (a) $\Phi(\snorm{2}{\boldsymbol{\mu}_+}) \geq \Omega( B / \log(1/B) )$, and (b) $\snorm{2}{\vec v} \leq \snorm{2}{\boldsymbol{\mu}_+} + O(1 / \sqrt{\log(1/B)})$.
% \begin{enumerate}[(a)]
%     \item 
%     \label{clm:(a)}
%     \item 
%     \label{clm:(b)}
% \end{enumerate}
Assuming the above claims, we immediately have that
\begin{align*}
    \Phi( \snorm{2}{\vec v} )
    \geq 
    \Phi\lp( \snorm{2}{\boldsymbol{\mu}_+} + O \lp(1 / \sqrt{\log(1/B)} \rp) \rp)
    \geq 
    \Omega(1) \; \Phi( \snorm{2}{\boldsymbol{\mu}_+}  )
    \geq \Omega(B / \log(1/B)) \;, 
\end{align*}
where in the first inequality we use (b), in the second inequality we use \Cref{clm:gaussian-cdf-bound} and
the fact that $\snorm{2}{\boldsymbol{\mu}_+} \leq O \lp( \sqrt{\log(1/B)} \rp)$, which is implied by (a), and in the last inequality we again use (a).
This shows that $\vec v$ is a $(0, \Omega(B/\sqrt{\log(1/B)}))$-good localization center.
It is easy to see that the output list contains some point that is $\eps^2$-close to $\vec v$, and closer to the origin than $\vec v$. Hence, it follows that the list contains some $\lp( \eps^2,  \Omega(B/\sqrt{\log(1/B)})\rp)$-good localization center.
This concludes the proof of \Cref{lem:certify-localization-center}.
\end{proof}

\begin{Ualgorithm}[h]
	\centering
	\fbox{\parbox{6in}{
			{\bf Input:} Sample access to a distribution $D$ over $\R^d \times \{\pm 1\}$;  
   tolerance parameter $\eps$.\\
   {\bf Output:} Reject or a list of vectors containing a good localization center.
\begin{enumerate}
% \item Estimate $\Pr_{ (\x, y) \sim D }\lp[ y = 1 \rp]$ up to accuracy $\eps/2$ and store the result as $\tilde B$. Declare failure if 
% $\tilde B < \eps/2$. 
\item Set $N = \poly(d/\eps)$. 
Draw $N$ \iid samples, and denote the set of samples as $S$.
\item Return an empty list if either the fraction of points labeled with $+1$ or $-1$ is less than $\eps / 2$. \label{line:minority-fraction} 
% \footnote{Note that the algorithm does not report violation of distribution assumption in this case. It just fails to find a good localization center.} 
\item Verify that $\E_{ (\x,y) \sim S } \lp[  \x \x^T\rp] \preccurlyeq 2 \vec I$. \label{line:cov-verify} 
\item Verify that $ \snorm{2}{\E_{  (\x, y) \sim S } \lp[ \x \rp]} < \eps $. \label{line:mean-verify}
\item Compute %the empirical means
$
\boldsymbol{\mu}_+ = \E_{ (\x, y) \sim S } [ \x | y = +1]
\, , \,
\boldsymbol  \mu_- = \E_{ (\x, y) \sim S } [\x | y = -1].
$
% \item Compute the empirical mean of the points with label $+1$ and denote it as $\mu_{+}$.
% (When there are not enough $+$ points, the algorithm fails to find good localization centers. Yet, this is fine as the learner can just output the halfspace that classifies everything to $-1$.)
\item Initialize an empty list $L$.
\item For $\boldsymbol{\mu} \in  \{ \boldsymbol{\mu}_+, \boldsymbol{\mu}_-\}$, do the following: 
\begin{enumerate}[leftmargin=*]
\item For each $(\x, y) \in S$, project it along the direction of $\boldsymbol\mu$ to obtain the new pair $ ( x', y  ) \in \R \times \{\pm 1\}$ where $x' = \boldsymbol \mu \cdot \x / \snorm{2}{\boldsymbol\mu}$. Denote the resulting set as $H$.
\item Verify that the empirical distribution over $H$ and $\normal(0, 1)$ are $\eps$-close in CDF distance.
\label{line:cdf-check}
% \item Verify the empirical distribution $H$ is $\lp(\eta, O \lp( \eta \sqrt{\log(1/\eta)} \rp) \rp)$ stable
% for $\eta = i \; \eps$ where $i \in [1/\eps]$.
% \label{line:stability-check}
\item Verify that removing at most $\eps$-fraction of points from $H$ changes the mean by at most $O( \eps \sqrt{\log(1/\eps)} )$.
\label{line:stability-check}
\item Add $ \big \{
\vec v^{(i)} :=
 i \eps^2 \boldsymbol  \mu /\snorm{2}{\boldsymbol \mu}  \big \}_i$, where $i \in  \{0\} \cup \ceil{\frac{ \snorm{2}{\boldsymbol \mu} + 1/\log(1/B) }{ \eps^2 / \log(1/\tilde B) }} $, to the list $L$.
 \end{enumerate}
 \item Run the algorithm from \Cref{lem:small-offset-center-search}, and add the localization centers found into $L$. Return $L$.
\end{enumerate}
}}
\caption{Find-Localization-Center} 
\label{alg1:localization-center}
\end{Ualgorithm}

% \vspace{-0.5em}

\subsection{Putting things together}
\label{sec:put-together}
% Our pseudocomain algorithm is given in \Cref{alg:main}. 
\snew{We describe our algorithm and its analysis at a high level in this section.
We first invoke the routine {Find-Localization-Center} from \Cref{lem:certify-localization-center} to obtain a list of candidate centers, which is guaranteed to contain some $\lp( \eps^2, \wt \Omega(B)\rp)$-localization center if the routine does not reject.
For each candidate center $\vec w$, we use $(\vec w, \sigma:=\min(\snorm{2}{\vec w}^{-1}, \sqrt{2}))$-rejection sampling
(\Cref{def:rejection-sampling}) to re-center the marginal distribution at $\vec w$, and then transform the marginal back into a standard Gaussian.
If $\vec w$ is indeed a good localization center,
by \Cref{lem:good-center}, the halfspace becomes 
$$h'(\x) =  \sgn \lp( \vec \Sigma^{1/2} \vec v^\ast \cdot \x / \| \vec \Sigma^{1/2} \vec v^\ast \|_2 + t' \rp)$$  after the transformation, where 
$\vec \Sigma = \vec I - (1 - \sigma^2) \vec w \vec w^\top / \snorm{2}{\vec w}^2$, and $t'$ is some threshold with size at most $\eps$. 
Moreover, the fraction of points that survive the rejection sampling is at least $\wt \Omega(B)$, which ensures that 
the fraction of outliers among the surviving 
points is at most $\wt O(\opt / B) = \wt O(\sqrt{\opt})$ under the assumption $B \gtrsim \sqrt{\opt}$ (otherwise, we can always output a constant halfspace in the end).
We then run {Nearly-Homogeneous-Halfspace-Testable-Learner}, which gives us an $\wt O( \opt/B )$-approximation to $\vec \Sigma^{1/2} \vec v^\ast / \| \vec \Sigma^{1/2} \vec v^\ast\|_2.$
By \Cref{lem:transformed-angle}, reverting the transformation $\vec \Sigma^{1/2}$ gives a 
vector $\hat {\vec v}$, which will be 
an $\wt O(\opt/B)$-approximation to $\vec v^\ast$ \footnote{We also need to show that the angle between $\vec w$ and $\vec v^\ast$ is not too large. The details of this argument can be found in Appendix~\ref{app:23}.}.
Since we search for the threshold in a brute-force manner, 
it is guaranteed that we will have some halfspace $\hat h$ whose parameters are all $\wt O(\sqrt{\opt})$-close to the optimal halfspace in our final hypothesis list.
Conditioned on that the Wedge-Bound algorithm passes, \Cref{lem:general-wedge-bound} then guarantees that $\hat h$ will have learning error at most $\wt O(\sqrt{\opt})$.
In the end, we simply draw some fresh samples, and output the halfspace with the smallest testing error. 
The detailed proof and pseudocode can be found in Appendix~\ref{app:23}.}

\section{Additional Remarks Regarding \texorpdfstring{\Cref{thm:testable-learning-agnostic}}{Lg}}
% \inote{Create an appendix for these two remarks}
We provide some additional remarks regarding our main theorem, addressing some of its limitations.
\begin{remark} \label{rem:broader-distr}
{\em It is natural to ask whether our algorithmic result can be generalized to hold for other structured distributions, e.g., for a large subclass of isotropic log-concave distributions. Such a generalization turns out to be possible for \emph{homogeneous halfspaces}; see \cite{gollakota2023tester}.
It is important to note, however, that obtaining polynomial-time agnostic learners for general halfspaces under non-Gaussian distributions is a challenging task --- even without the testable requirement. 
In particular, the only known efficient agnostic learners for general halfspaces that achieve dimension-independent error are the ones from \cite{DKS18a, diakonikolas2022learning}, 
which both work only under the Gaussian distribution.}
\end{remark}

\begin{remark}\label{rem:o-opt}
{\em \Cref{thm:testable-learning-agnostic} achieves error of $\wt O(\sqrt{\opt})+\eps$. 
On the other hand, error of $O(\opt)+\eps$ can be achieved in the same setting for 
homogeneous halfspaces. We discuss here a technical barrier of our approach 
preventing further improvements. A critical ingredient in our algorithm is a procedure which robustly estimates the mean of samples sharing 
the same label (in a testable way), 
in order to extract directional information about the weight vector of the optimal halfspace. Our current method, relying on certifying bounded 
second moments of the samples, ceases to work when the amount of corruption approaches $\sqrt{B}$, where $B$ is the probability mass of the points having that label. It is plausible that certifying boundedness of higher moments is an avenue for progress here.}
\end{remark}

% We then run the Nearly-Homogeneous-Testable-Learn (\Cref{prop:small-offset-halfspace}) on the resulting distribution to obtain a candidate defining vector $\vec v$. 
% Conditioned on that the original marginal distribution $D_\x$ passes the {Wedge-Bound} test in the direction of $\vec v$, we add $\vec v$ to our hypothesis list (and reject otherwise).
% In the end, we compute the empirical errors of our hypothesis, and pick the best one.

%\newpage
\bibliographystyle{alpha}
\bibliography{mydb}

\newcommand{\etalchar}[1]{$^{#1}$}
\begin{thebibliography}{GKSV23b}

\bibitem[ABL17]{ABL17}
P.~Awasthi, M.~F. Balcan, and P.~M. Long.
\newblock The power of localization for efficiently learning linear separators
  with noise.
\newblock {\em J. {ACM}}, 63(6):50:1--50:27, 2017.

\bibitem[Dan15]{Daniely15}
A.~Daniely.
\newblock A {PTAS} for agnostically learning halfspaces.
\newblock In {\em Proceedings of The 28\textsuperscript{th} Conference on
  Learning Theory, {COLT} 2015}, pages 484--502, 2015.

\bibitem[Dan16]{Daniely16}
A.~Daniely.
\newblock Complexity theoretic limitations on learning halfspaces.
\newblock In {\em Proceedings of the 48\textsuperscript{th} Annual Symposium on
  Theory of Computing, {STOC} 2016}, pages 105--117, 2016.

\bibitem[DK23]{diakonikolas2023algorithmic}
I.~Diakonikolas and D.~M. Kane.
\newblock {\em Algorithmic high-dimensional robust statistics}.
\newblock Cambridge university press, 2023.

\bibitem[DKK{\etalchar{+}}21]{DKKTZ21}
I.~Diakonikolas, D.~M. Kane, V.~Kontonis, C.~Tzamos, and N.~Zarifis.
\newblock Agnostic proper learning of halfspaces under gaussian marginals.
\newblock In {\em Proceedings of The 34\textsuperscript{th} Conference on
  Learning Theory, {COLT}}, 2021.

\bibitem[DKK{\etalchar{+}}23]{diakonikolas2023efficient}
I.~Diakonikolas, D.~M Kane, V.~Kontonis, S.~Liu, and N.~Zarifis.
\newblock Efficient testable learning of halfspaces with adversarial label
  noise.
\newblock In {\em Advances in Neural Information Processing Systems}, 2023.

\bibitem[DKMR22]{DKMR22b}
I.~Diakonikolas, D.~M. Kane, P.~Manurangsi, and L.~Ren.
\newblock Cryptographic hardness of learning halfspaces with massart noise.
\newblock In {\em Advances in Neural Information Processing Systems}, 2022.

\bibitem[DKPZ21]{DKPZ21}
I.~Diakonikolas, D.~M. Kane, T.~Pittas, and N.~Zarifis.
\newblock The optimality of polynomial regression for agnostic learning under
  gaussian marginals in the {SQ} model.
\newblock In {\em Proceedings of The 34\textsuperscript{th} Conference on
  Learning Theory, {COLT}}, 2021.

\bibitem[DKR23]{DKR23}
I.~Diakonikolas, D.~M. Kane, and L.~Ren.
\newblock Near-optimal cryptographic hardness of agnostically learning
  halfspaces and relu regression under gaussian marginals.
\newblock In {\em ICML}, 2023.

\bibitem[DKS18]{DKS18a}
I.~Diakonikolas, D.~M. Kane, and A.~Stewart.
\newblock Learning geometric concepts with nasty noise.
\newblock In {\em Proceedings of the 50\textsuperscript{th} Annual {ACM}
  {SIGACT} Symposium on Theory of Computing, {STOC} 2018}, pages 1061--1073,
  2018.

\bibitem[DKTZ20]{DKTZ20c}
I.~Diakonikolas, V.~Kontonis, C.~Tzamos, and N.~Zarifis.
\newblock Non-convex {SGD} learns halfspaces with adversarial label noise.
\newblock In {\em Advances in Neural Information Processing Systems,
  {NeurIPS}}, 2020.

\bibitem[DKTZ22]{diakonikolas2022learning}
I.~Diakonikolas, V.~Kontonis, C.~Tzamos, and N.~Zarifis.
\newblock Learning general halfspaces with adversarial label noise via online
  gradient descent.
\newblock In {\em International Conference on Machine Learning}, pages
  5118--5141. PMLR, 2022.

\bibitem[DKZ20]{DKZ20}
I.~Diakonikolas, D.~M. Kane, and N.~Zarifis.
\newblock Near-optimal {SQ} lower bounds for agnostically learning halfspaces
  and {ReLUs} under {G}aussian marginals.
\newblock In {\em Advances in Neural Information Processing Systems,
  {NeurIPS}}, 2020.

\bibitem[GGK20]{GGK20}
S.~Goel, A.~Gollakota, and A.~R. Klivans.
\newblock Statistical-query lower bounds via functional gradients.
\newblock In {\em Advances in Neural Information Processing Systems,
  {NeurIPS}}, 2020.

\bibitem[GKK23]{Gollakota2022}
A.~Gollakota, A.~Klivans, and P.~Kothari.
\newblock A moment-matching approach to testable learning and a new
  characterization of rademacher complexity.
\newblock In {\em Proceedings of the 55th Annual ACM Symposium on Theory of
  Computing}, pages 1657--1670, 2023.

\bibitem[GKSV23a]{GKSV23}
A.~Gollakota, A.~R. Klivans, K.~Stavropoulos, and A.~Vasilyan.
\newblock An efficient tester-learner for halfspaces.
\newblock {\em arXiv}, 2023.
\newblock Conference version to appear in ICLR'24.

\bibitem[GKSV23b]{gollakota2023tester}
A.~Gollakota, A.~R Klivans, K.~Stavropoulos, and A.~Vasilyan.
\newblock Tester-learners for halfspaces: Universal algorithms.
\newblock In {\em Advances in Neural Information Processing Systems}, 2023.

\bibitem[Hau92]{Haussler:92}
D.~Haussler.
\newblock {Decision theoretic generalizations of the PAC model for neural net
  and other learning applications}.
\newblock {\em Information and Computation}, 100:78--150, 1992.

\bibitem[KKMS08]{KKMS:08}
A.~Kalai, A.~Klivans, Y.~Mansour, and R.~Servedio.
\newblock Agnostically learning halfspaces.
\newblock {\em SIAM Journal on Computing}, 37(6):1777--1805, 2008.
\newblock Special issue for FOCS~2005.

\bibitem[KLS09]{KLS:09jmlr}
A.~Klivans, P.~Long, and R.~Servedio.
\newblock {Learning Halfspaces with Malicious Noise}.
\newblock {\em Journal of Machine Learning Research}, 10:2715--2740, 2009.

\bibitem[KSS94]{KSS:94}
M.~Kearns, R.~Schapire, and L.~Sellie.
\newblock {Toward Efficient Agnostic Learning}.
\newblock {\em Machine Learning}, 17(2/3):115--141, 1994.

\bibitem[Ros58]{Rosenblatt:58}
F.~Rosenblatt.
\newblock The {P}erceptron: a probabilistic model for information storage and
  organization in the brain.
\newblock {\em Psychological Review}, 65:386--407, 1958.

\bibitem[RV23]{RV22a}
R.~Rubinfeld and A.~Vasilyan.
\newblock Testing distributional assumptions of learning algorithms.
\newblock In {\em STOC}, 2023.

\bibitem[Tie23]{Tiegel22}
S.~Tiegel.
\newblock Hardness of agnostically learning halfspaces from worst-case lattice
  problems.
\newblock In {\em COLT}, 2023.

\bibitem[Val84]{val84}
L.~G. Valiant.
\newblock A theory of the learnable.
\newblock In {\em Proc.\ 16th Annual ACM Symposium on Theory of Computing
  (STOC)}, pages 436--445. ACM Press, 1984.

\end{thebibliography}
\appendix

\newpage

\section*{Appendix}

\section{Testable Learner for Nearly Homogeneous Halfspaces}\label{app:secgeneral}
We restate and show the following:

\begin{proposition}
Let $\eps \in (0,1)$.
Let $D$ be a distribution over $\R^d \times \{\pm 1\}$ and $h$ be a halfspace that achieves error $\opt$ under $D$ with $h(\x) = \sgn(\vec v^\ast\cdot \x + t)$ where $\abs{t} < \eps$.
Then, given \iid sample access to $D$, there exists an efficient tester-learner which either reports $D_\x$ is not Gaussian or outputs a vector $\vec v$ satisfying $\snorm{2}{\vec v - \vec v^\ast} \leq O(\opt) + \eps$.
\end{proposition}
\begin{proof}
    The algorithm of \Cref{prop:small-offset-halfspace} is identical to the algorithm of \cite{diakonikolas2023efficient}, but we need to argue that the small offset does not make the algorithm to fail. If the distribution was Gaussian, then trivially the $\pr[\sgn(\vec v^\ast\cdot \x + t)\neq \sgn(\vec v^\ast\cdot \x)]=O(\eps)$, therefore, we could just assume that there exists a homogeneous halfspace with error $\opt+O(\eps)$. Unfortunately, in our setting we cannot do that. The reason is that, we do not have any guarantee that all the bands of the form $\{|\vec u\cdot \x|\leq \eps\}$ have mass of order $\eps$ for all the vectors $\vec u$.

The whole approach of \cite{diakonikolas2023efficient} boils down into calculating the $\E[y\x]$ up to a good accuracy. This is done in Proposition 2.1 of this work. We need to show that if the optimal halfspace is of the form $\sign(\vec v^\ast\cdot\x+t)$ with $|t|\leq 10$, then we can robustly estimate $\E[y\x]$. We first need to show that small $|t|$ does not change the $\E[y\x]$ by a lot. We use the following fact

    \begin{fact}[see, e.g.,~Lemma 4.3 of~\cite{DKS18a}] \label{lem:unbiased-chow}
Let $\vec v$ be a unit vector and $h(\vec x) = \sgn(\vec v \cdot \vec x+t)$ 
be the corresponding halfspace.
If $\vec x$ is drawn from $\normal(\vec 0, \vec I)$, 
then  we have that 
$\E_{\vec x \sim \normal(\vec 0, \vec I)} \lp[ h(\vec x)  \vec x\rp] = 2 G(t) ~ \vec v$, where $G(\cdot)$ is the pdf of the standard Gaussian.
\end{fact}

Note that from \Cref{lem:unbiased-chow}, we can see that if $|t'|\leq 10$ then $G(t')\geq C$, for some sufficiently small absolute constant $C>0$. Therefore, the proof of Proposition 2.1 remains the same as it only changes by a constant factor the calculation of $\E[y\x]$. 

Furthermore, note that in each iteration, the algorithm of \cite{diakonikolas2023efficient} changes the covariance matrix as $\vec \Sigma =\vec I-(1-\sigma^2)\vec w\vec w^\top$  for some $\sigma\geq \Omega(\eps)$ via rejection sampling. This means that norm of the vector $\vec v^\ast$ after the application of of the operator $\vec \Sigma^{1/2}$ will be at least $\sigma$, hence after normalization we have that $\sign(\vec \Sigma^{1/2}\vec v^\ast\cdot\x +t)=\sign((\vec \Sigma^{1/2}\vec v^\ast\cdot\x +t)/\|\vec \Sigma^{1/2}\vec v^\ast\|_2)$, which means that $|t|$ will be increased to at most $|t|/\sigma$. Moreover, note that $|t|\leq \eps$ and $\sigma=\Omega(\eps)$, hence $|t'|=|t|/\sigma\leq O(1)$ and we can choose the constants so that $|t'|\leq 10$.
Hence, the error $\snorm{2}{ \vec  v - \vec v^\ast }$ steadily shrinks by a constant factor in each iteration until we have $\snorm{2}{ \vec  v - \vec v^\ast } = \Theta( \max( \eps, \opt ) )$.
This completes the proof.
% Finally, note that from \Cref{lem:general-wedge-bound}, we have that if $\|\vec v-\vec v^\ast\|_2\leq \delta$ and $|t|\leq \eps$, then $\pr[\sgn(\vec v\cdot \x)\neq \sgn(\vec v^\ast\cdot \x)]=O(\delta+\eps)$. This completes the proof.
\end{proof}

\section{Omitted Proofs for General to Near-Homogeneous Reduction}\label{app:21}
% \section{Omitted Proofs from \cref{sec:good-localization-def}}
\subsection{Proof of Properties of Good Localization Center
(\texorpdfstring{\Cref{lem:good-center}}{Lg})}
\begin{proof}
By the definition of an $(\alpha, \beta)$-good localization center, $\vec w$ satisfies
$
\Phi( \snorm{2}{\vec w} )
    \geq \beta.
$
Note that $\Phi( \snorm{2}{\vec w} )$ can be bounded from above by
$\frac{1}{ \sqrt{2 \pi} } \; \frac{1}{ \snorm{2}{\vec w} }  \; \exp(  - \snorm{2}{\vec w}^2/2  )$.
It then follows that 
$$
\exp( -\snorm{2}{\vec w}^2/2 ) \geq C \; \snorm{2}{\vec w} \beta.
$$
From here one can conclude that
$
\snorm{2}{\vec w} \leq O \lp(  \sqrt{\log(1/\beta)} \rp).
$

The form of the rejection sampling distribution conditioned on acceptance follows from \cite{DKS18a}. To analyze the acceptance probability, we first assume $\snorm{2}{\vec w}> \sqrt{2}$. 
If we perform $(\vec w, 1/\snorm{2}{\vec w})$-rejection sampling on the standard Gaussian distribution, the acceptance probability will be
\begin{align*}
{ \snorm{2}{\vec w}^{-1} }
    \; \exp\lp(  - \frac{\snorm{2}{\vec w}^2} { 2 \; (1 - \snorm{2}{\vec w}^{-2})} \rp) 
&=     
{ \snorm{2}{\vec w}^{-1} }
    \; \exp \lp(  
    -
    \frac{ \snorm{2}{\vec w}^2 }{2}
    \;
    \lp(1 +
    \frac{\snorm{2}{\vec w}^{-2}}{(1 - \snorm{2}{\vec w}^{-2})}
     \rp)
     \rp)\\
&=
{ \snorm{2}{\vec w}^{-1} }
\;
\exp \lp(  - \frac{\snorm{2}{\vec w}^2}{2}\rp)
\;
\exp \lp( 
-\frac{1}{2 (1 - \snorm{2}{\vec w}^{-2})}
\rp).
\end{align*}
Recall that in this case we assume $\snorm{2}{\vec w}$ is within the range $[\sqrt{2}, \infty)$. Hence, $\snorm{2}{\vec w}^{-2}$ is within the range $(0, 1/2]$. As a result, $\exp \lp( 
-\frac{1}{2 (1 - \snorm{2}{\vec w}^{-2})} \rp)$ is bounded from above and below by constants.
Hence, we can bound from below the overall acceptance probability by 
$$
\frac{1}{ \snorm{2}{\vec w} }
    \; \exp\lp(  - \snorm{2}{\vec w}^2 / \lp( 2 \; (1 - \snorm{2}{\vec w}^{-2})\rp) \rp) 
\geq
C \; \Phi(\norm{\vec w}) \geq 
C \beta.
$$
Now, suppose $\snorm{2}{\vec w} < \sqrt{2}$. 
We now perform rejection sampling with parameter $\sigma = \sqrt{1/2}$.
The acceptance probability is then
$
2 \; \exp\lp(  - \snorm{2}{\vec w}^2 \rp)
\geq \Omega(1).
$
This completes the argument of Property~(1).

After rejection sampling, the standard Gaussian becomes
$\normal( \vec w, \vec \Sigma )$,
where $\vec \Sigma = \vec I - (1 - \sigma^2) \vec w \vec w^T / \snorm{2}{\vec w}^2$ for 
$\sigma = \min(\sqrt{1/2}, \snorm{2}{\vec w}^{-1}) > \Omega \lp(1 / \sqrt{\log(1/\beta)} \rp)$.
Recall the original halfspace is given by $h(\x) = 
\sgn(  
\vec v \cdot \x
+t )$.
After we transform the space to make $\normal(\vec w, \vec \Sigma)$ isotropic, the halfspace then becomes
$$
h'(\x) = 
\sgn( \vec \Sigma^{1/2} \vec v^\ast \cdot \x + \vec v^\ast \cdot \vec w + t^\ast).
$$
Now we analyze the $\ell_2$ norm of the new defining vector and offset separately.
For the new defining vector,
we decompose $\vec v^\ast$ into its component in $\vec w$ and the orthogonal part, i.e.
$\vec v^\ast = a \vec w / \snorm{2}{\vec w} + b \vec u $  where $a^2 + b^2 = 1$.
Then, it is not hard to see that
$  \vec \Sigma^{1/2} \vec v = 
a \vec w / \snorm{2}{\vec w} + \sigma \; b \vec u   $.
Hence, the $\ell_2$ norm is at least $\max(a, \sigma b) \geq \sigma/\sqrt{2} > \Omega(1/\log(1/\beta)) $.
For the threshold, we note that by the definition of the localization center, $\vec w$ is at most $\alpha$ far from the halfspace $h$.
Therefore, it holds $\abs{\vec v^\ast \cdot \vec w + t^\ast} \leq \alpha$.
Combining our analysis then gives the new halfspace $h'$ is at most $O\lp(\alpha \sigma^{-1}\rp)
= O\lp( \alpha \sqrt{\log(1/\beta)} \rp)$ far from the origin.
This concludes the proof of Property (2), and also \Cref{lem:good-center}.
\end{proof}

\subsection{Proof of Transformation Error Bound
(\texorpdfstring{\Cref{lem:transformed-angle}}{Lg})}

\begin{proof}
For convenience, we define
$\lambda = \snorm{2}{\vec \Sigma^{1/2} \vec v^{\ast}}$.
We can decompose 
$\vec v^\ast$ as
$\vec v^\ast = 
a \; \vec w +b \; \vec y
$
for some unit vector $\vec y$ orthogonal to $\vec w$ and positive coefficients $a,b$ satisfying $a^2 +b^2 = 1$.
Notice that we have
$(1-a)^2 + b^2 = \snorm{2}{\vec w - \vec v^\ast}\leq \beta^2$.
This implies that $b$ is bounded from above by $\beta$.
Consequently, since $a^2 +b^2 = 1$, we have
\begin{align} \label{eq:a-bound}
 a \geq 1 - \beta^2/2.
\end{align}
% On the other hand, $b$ is upper bounded by $\beta$.
Note that the operator $\vec \Sigma^{1/2}$ scales  any vector along the direction of $\vec w$ 
by a factor of $\sigma$
while leaving any vector orthogonal to $\vec w$ unchanged.
Therefore, we have
$
\vec \Sigma^{1/2} \vec v^\ast
$, 
which further implies that $\lambda=\snorm{2}{\vec \Sigma^{1/2} \vec v^{\ast}} \leq a\sigma + b$.

Since the $\ell_2$ distance between $\vec v$ and $\lambda^{-1} \vec \Sigma^{1/2} \vec v^\ast$ is at most $\delta$, we can write
$$
\vec v = \lambda^{-1} \vec \Sigma^{1/2} \vec v^\ast + c \; \vec w + d \; \vec z.
$$
for some unit vector $\vec z$  orthogonal to $\vec w$ and 
positive coefficients $c,d < \delta$.\\
Multiplying both sides by $\vec \Sigma^{-1/2}$ then gives
$$
\vec \Sigma^{-1/2} \vec v
= \lambda^{-1} \vec v^\ast
+ c \; \sigma^{-1} \; \vec w + d \; \vec z.
$$
We can substitute the decomposition of $\vec v^\ast$ into the equation.
After some algebra, we get that
\begin{align*}
\vec \Sigma^{-1/2} \vec v
&= \lp( \lambda^{-1} + c \; \sigma^{-1} \; a^{-1}  \rp) \; a \; \vec w
+ \lp( \lambda^{-1} + c \; \sigma^{-1} \; a^{-1}  \rp) \; b \; \vec y
- c \; \sigma^{-1} \; a^{-1} b \vec y + d \; \vec z. \\
&= \lp( \lambda^{-1} + c \; \sigma^{-1} \; a^{-1}  \rp) \vec v^\ast
- c \; \sigma^{-1} \; a^{-1} b \; \vec y + d \; \vec z.
\end{align*}
Now denote $\xi := \lambda^{-1} + c \; \sigma^{-1} a^{-1}$.
Note that since $\lambda <a \; \sigma + \beta$, we have $\xi^{-1} <a \; \sigma + \beta$ as well.
We can then divide both sides by $\xi$ and merge the terms involving $\vec y, \vec z$, which gives us
$$
\xi^{-1} \vec \Sigma^{-1/2} \vec v
= \vec v^\ast +  \rho  \vec u \, ,
$$
where $\vec u$ is a unit vector and 
$\rho$ is a positive number bounded by
\begin{align*}
\xi^{-1} \lp( c \sigma^{-1} a^{-1} b + d\rp)
\leq 
(a \sigma + \beta) \;
( c \sigma^{-1} a^{-1} b + d )
\leq (a \sigma + \beta) \;
\lp( \delta \sigma^{-1} \frac{\beta}{1 - \beta^2/2} + \delta  \rp).
\end{align*}
% $ O \lp(1\rp) \; (\sigma +\beta) \; (\delta +\delta \sigma^{-1} \beta) $.
Note that the right hand side is exactly the bound we want for $\snorm{2}{ 
\vec \Sigma^{1/2}
\vec v / \snorm{2}{\vec \Sigma^{1/2}
\vec v} -
\vec v^\ast
}$ up to some constant.
If the right hand side is $ \Omega(1) $, the conclusion is trivial since the distance between two unit vectors is always upper bounded by a constant.
Otherwise, 
since $\xi^{-1} \vec \Sigma^{-1/2} \vec v$ is $\rho$-close to $\vec v^\ast$ in $\ell_2$ distance and $\rho = o(1)$, it follows that the vector, after being normalized to have norm $1$, is still $O(\rho)$ closed to $\vec v^\ast$ in $\ell_2$ distance. 
Our result then follows from the observation that  
$$
\xi^{-1} \vec \Sigma^{-1/2} \vec v
/\snorm{2}{\xi^{-1} \vec \Sigma^{-1/2} \vec v} =  \vec \Sigma^{-1/2} \vec v
/\snorm{2}{ \vec \Sigma^{-1/2} \vec v}.
$$
This concludes the proof of \Cref{lem:transformed-angle}.
\end{proof}

\subsection{Proof of General Wedge Bound (\texorpdfstring{\Cref{lem:general-wedge-bound}}{Lg})}
We provide the pseudocode of the Wedge-Bound algorithm below for completeness.

\begin{Ualgorithm}[H]
	\centering
	\fbox{\parbox{5.3in}{
			{\bf Input:} Sample access to a distribution $D_{\vec x}$ over $\R^d$; tolerance parameter $\eta>0$; unit vector $\vec v \in \R^d$;
failure probability $\tau \in (0,1)$.\\
{\bf Output:} Certifies the (conditional) moments of $D$ approximately match with $\normal(\vec 0, \vec I)$.
\begin{enumerate}[leftmargin=*]
\item Set $B = \ceil{\sqrt{\log(1/\eta)} / \eta}$.
\item Let $\widetilde D$ be the empirical distribution obtained by drawing $\poly(d, 1/\eta)  \log(1/\tau)$ %many 
samples from $D_{\vec x}$.
\item For integers $-B-1 \leq i \leq B$, define $E_i$ to be the event  that $\{\vec v \cdot \vec x \in [ i\eta , (i+1)  \eta ]\}$ and $E_{B+1}$ to be the event  that $\{|\vec v \cdot \vec x| \geq \sqrt{\log(1/\eta)}\}$.
\item \label{line:probability-check} Verify that
$
\sum_{i=-B-1}^{B+1}
\abs{ \Pr_{ \normal(\vec 0, \vec I) } \lp[ E_i\rp] - \Pr_{ \widetilde D } \lp[ E_i\rp]}\
\leq \eta.
$
\item Let $S_i$ the distribution of $\widetilde D$ conditioned on $E_i$ and $S_i^{\perp}$ be $S_i$ projected on the subspace orthogonal to $\vec v$.
% \item  \label{line:moment-check} For each $i$, verify that the first and second moments of $S_i^{\perp}$ match with those of a standard Gaussian up to error $\eta$: verify that 
% \[
% \Big \|\E_{\x \sim S_i^\perp}[\x] \Big\|_2 \leq \eta \qquad
% \E_{\x \sim S_i^\perp}[\x \x^\top] \preccurlyeq 2 \vec I
% \,.
% \]
\item \label{line:moment-check} For each $i$, verify that $S_i^{\perp}$ 
has bounded covariance, i.e., check that
\(
% \Big \|\E_{\x \sim S_i^\perp}[\x] \Big\|_2 \leq \eta \qquad
\E_{\x \sim S_i^\perp}[\x \x^\top] \preccurlyeq 2 \vec I
\,.
\)

\end{enumerate}}}
\vspace{0.2cm}
\caption{Wedge-Bound} \label{alg:wedge-bound}
\end{Ualgorithm}

\begin{proof}[Proof of \Cref{lem:general-wedge-bound}]
We start with the following claim: 

\begin{claim}
\label{clm:wedge}
Suppose $D$ is a distribution passing the Wedge Bound Test with tolerance $\eta$ along the direction $\vec v$.
Let $a,b$ be two positive number satisfying $a^2 + b^2 = 1$ and $b <\delta$.
Let $\vec u$ be a unit vector orthogonal to $\vec v$.
Then it holds 
$$
\Pr [ -b \vec u \cdot \vec x > a \vec v \cdot \vec x + t, a \vec v \cdot \vec x + t > 0]  \, , \,
\Pr [ b \vec u \cdot \vec x > -a \vec v \cdot \vec x - t, -a \vec v \cdot \x - t > 0]\leq O(\delta +\eta).
$$
\end{claim}
\begin{proof}
For convenience, we define $x := - \vec u \cdot \vec x$, 
and $y := \vec v \cdot \vec x$. We have
\begin{align*}
&\Pr [ -b x > a y + t, a y + t > 0]\\
&\leq 
\Pr [ 0 < a y + t < \delta ]
+ 
\sum_{i=1}^{\infty}
\Pr [ -b x > i \delta, i \delta < a y + t < (i+1) \delta]  \\
&\leq O(\delta + \eta)
+ O(\delta + \eta) \; \sum_{i=1}^{\infty} \frac{1}{i^2}
\leq O(\delta + \eta) \, ,
\end{align*}
where the first inequality uses the law of total probability, the second inequality follows from that we verify that the cumulative density function of $y$ is  $\eta$-close to that of the standard Gaussian in Kolmogorov distance, and that the distribution of $x$ conditioned on $i \delta < a y + t < (i+1) \delta$ has its second moment bounded from above by some constant, the last inequality follows from that the series $\sum_{i=1}^{\infty} \frac{1}{i^2}$ is converging.
The argument for bounding the term $\Pr [ b \vec u \cdot \vec x > -a \vec v \cdot \vec x - t, -a \vec v \cdot \x - t > 0]$ is symmetric, and we hence omit it here.
This concludes the proof of \Cref{clm:wedge}.
\end{proof}
We first decompose $\vec v^\ast$ as
$\vec v^\ast = a \; \vec v \cdot \x  + b \; \vec u \cdot \x$ where 
$\vec u$ is a unit vector orthogonal to $\vec v$ and $a,b$ are two positive coefficients satisfying $a^2 +b^2 = 1$ and $b < \delta$.
Essentially, there are two cases where $h(\x)$ and $h^\ast(\x)$ could differ.
In the first case, we have
$a \; \vec v \cdot \x + b \; \vec u \cdot \x + t^\ast < 0 $ and $ \vec v \cdot \x + t > 0 $.
The first inequality can be rewritten as
$
-b \; \vec u \cdot \x > a \; \vec v \cdot \x + t^\ast
$
Using Claim~\ref{clm:wedge}, we have
$$
\Pr_{\x \sim D} \lp[ 
-b \; \vec u \cdot \x > a \; \vec v \cdot \x + t^\ast \, , \, a \; \vec v \cdot \x + t^\ast \geq 0
\rp] \leq O(\delta +\eta).
$$
On the other hand, since we must also have
$\vec v \cdot \x +t >0$, the probability that $\vec v \cdot \x +t >0$ and $a \; \vec v \cdot \x + t^\ast < 0$ are both true cannot be too large.
In particular, if both of them are true, we must have
$ t^\ast / a< -\vec v \cdot \x <t $.
Since $\abs{t - t^\ast} < \eta$, and $a$ is of order $1 \pm O(\delta)$.
It follows this happens with probability at most $O(\eta +\delta)$.
Altogether, we then have
\begin{align*}
&\Pr_{ \x \sim D }
\lp[ 
a \; \vec v \cdot \x + b \; \vec u \cdot \x + t^\ast < 0  \, , \, \vec v \cdot \x + t > 0
\rp] \\
&\leq 
\Pr_{\x \sim D} \lp[ 
-b \; \vec u \cdot \x > a \; \vec v \cdot \x + t^\ast \, , \, a \; \vec v \cdot \x + t^\ast \geq 0
\rp] 
+ 
\Pr_{\x \sim D} \lp[ 
\vec v \cdot \x +t >0  \, , \,
a \; \vec v \cdot \x + t^\ast < 0
\rp] \\
&\leq O(\delta +\eta).
\end{align*}

In the second case, we have
$a \; \vec v \cdot \x + b \; \vec u \cdot \x + t^\ast > 0$ and $\vec v \cdot \x + t < 0 $.
Similarly, we can rewrite the first inequality as 
$
b \; \vec u \cdot \x >-a \; \vec v \cdot \x - t^\ast.
$
Now, since $- \vec v \cdot \x -t >0$ and $t >t^\ast$, we must have
$-a \; \vec v \cdot \x - t^\ast >0$.
Hence, we can just use \Cref{clm:wedge} to conclude that this happens with probability at most $O(\eta +\delta)$.
Combining the bound on the probability mass in the two cases then concludes the proof of \Cref{lem:general-wedge-bound}.
\end{proof}

\section{Omitted Proofs for Good Localization Center Search}\label{app:22}

% \subsection{Omitted Lemmas and Claims from \Cref{sec:find-center}}
\subsection{Soundness of the Tests in \texorpdfstring{\Cref{alg1:localization-center}}{Lg}}
In this subsection, we show that \Cref{alg1:localization-center} is sound, i.e., it rarely falsely rejects $D_\x$ when we have $D_\x = \normal(\vec 0, \vec I)$.
\begin{lemma}[Soundness of Tests] \label{lem:test-soundness}
Let $D$ be a distribution over $\R^d \times \{\pm 1\}$. 
Assume $D_\x = \normal(\vec 0, \vec I)$.
Then Algorithm~\ref{alg1:localization-center} does not reject on Lines~\ref{line:cov-verify}, ~\ref{line:mean-verify},~\ref{line:cdf-check}~and~\ref{line:stability-check} with high constant probability.
\end{lemma}

\begin{proof}
We first show that if $D_\x$ is indeed $\normal(\vec 0, \vec I)$, the algorithm does not report violation of the distributional assumptions with high constant probability.
Lines~\ref{line:cov-verify} and~\ref{line:mean-verify} verify that the first and second moments are close to the ones of the standard normal; it follows from the fact that if $D_\x$ is the standard Gaussian, then by drawing $N \gtrsim d^2 / \eps^2$ \iid samples from $D_\x$, the first and the second empirical moments should be close to those of $\normal(\vec 0, \vec I)$ with high constant probability. 
For Line~\ref{line:cdf-check}, we remark that if $D_\x$ is standard Gaussian, then $H$ is empirical distribution of $\normal(0,1)$ and $N$ \iid samples from $D_\x$ are enough to test it with high constant probability.
% ------------
% Hence, if the projection direction $\boldsymbol{\mu}_+$ were not chosen adaptively based on the samples, the passing of the test would follows from the Dvoretzky–Kiefer–Wolfowitz inequality.--------
To argue the soundness of the test
we argue that  the uniform distribution over $\{ \vec v \cdot \vec x \mid \vec x \in S  \}$ is close to $\normal(0, 1)$ in Kolmogorov distance for every $\vec v$.
We note that the quantity can be expressed as
$$
\sup_{ \vec v \in \R^d, b \in \R }
\abs{
\Pr_{ \vec x \sim S }
\lp[ \vec v \cdot \vec x \geq b \rp]
-
\Pr_{ \vec x \sim \normal(\vec 0, \vec I) }
\lp[ \vec v \cdot \vec x \geq b \rp]
}.
$$
Hence, by the VC-inequality, with high constant probability, the above is bounded from above by
$\eps$ when we take $N \gtrsim d / \eps^2$ many samples.
Conditioned on the event that Line~\ref{line:cdf-check} does not reject, we verify that the test passes Line~\ref{line:stability-check} with high constant probability.
For Line~\ref{line:stability-check}, we note that if one takes $N \gtrsim d / \eps^2$ many samples from $\normal(\vec 0, \vec I)$, the empirical distribution will satisfy the following stability property (see \cite{diakonikolas2023algorithmic} e.g.)
\begin{align}
\label{eq:mean-deviation}
\abs{  
\frac{1}{|S'|}
\sum_{ \vec x \in S' }
\boldsymbol{\mu}_+ \cdot \vec x
} \leq O \lp( \eps \sqrt{\log(1/\eps)} \rp) \, ,  
\end{align}
for every $S' \subseteq S$ with $|S'| \eps (1 - \eps) |S|$.
On the other hand, since Line~\ref{line:mean-verify} passes, we immediately have that
$ \E_{ x \sim H }[x] \leq \eps $ and
combining it with \Cref{eq:mean-deviation}, it follows that the algorithm does not reject with high constant probability.
This concludes the proof of \Cref{lem:test-soundness}.
\end{proof}

\subsection{Proof of \texorpdfstring{\Cref{lem:mu-v-distance}}{LG}}

\begin{proof}[Proof of \Cref{lem:mu-v-distance}]
Without loss of generality, we assume that the tail points have the $+$ label.
Recall that $\boldsymbol \mu$ is the mean of the tail points among samples, and $\mathbf v$ is the intersection between the separating hyperplane of the halfspace and the ling along $\boldsymbol \mu$.
Suppose the optimal halfspace is given by 
$ h(\x) = \sgn(\vec v^\ast \cdot\x + t^\ast) $.
We first argue that 
\begin{align}
\label{eq:the-right-one}
\vec v^\ast \cdot \lp( \vec v - \boldsymbol{\mu}\rp) \leq O(1/\log(1/B)).    
\end{align}
Since $\vec v$ is defined to be the intersection between the hyperplane and some line passing the origin, we always have
$\vec v \cdot \vec v^\ast = |t^\ast|$.
Let $\boldsymbol {\bar \mu}$ be the 
empirical mean over the points with label $+1$ when there are no outliers, i.e., all points are labeled by $h(\x)$ (in particular, $\boldsymbol{\mu}$ can be viewed as a noisy estimation of $\boldsymbol {\bar \mu}$).
Then we have 
$\vec v^{\ast} \cdot \boldsymbol {\bar \mu} \geq |t^\ast| = \vec v \cdot \vec v^\ast$.
Thus, it suffices to show that
% We now bound the distance between $\vec v^\ast \cdot\boldsymbol {\bar \mu_+}$ and $\vec v^\ast \cdot\vec {\boldsymbol{\mu}_+}$.
% In particular, we claim that
\begin{equation}
\label{eq:outlier-mean-deviation}
\vec v^\ast \cdot \lp( \boldsymbol {\bar \mu} - \boldsymbol{\mu}\rp)
\leq \frac{1}{ \log(1/B) }.
\end{equation}
Let $\bar B$ be the fraction of samples within $S$ such that $h(\x) = +1$, and $\tilde B$ be the fraction of samples within $S$ such that $y = +1$.
Recall that $B$ is the fraction of tail points. Thus, with high probability, it holds that $|\tilde B - B| \leq \epsilon / 10$.
We also have that $\abs{\bar B - \tilde B} \leq \opt$ since the adversary can at most flip the labels of $\opt$ fraction of points.
Moreover, conditioned on that the algorithm does not declare reject in Line~\ref{line:minority-fraction}, we immediately have that
$\tilde B > \epsilon/2$.
By the assumption of the lemma, we have that $B \gtrsim \sqrt{\opt} \log^2(1/\opt) + \epsilon$. 
This implies that $\bar B, \tilde B \geq \Omega(B)$.
In order to show \Cref{eq:outlier-mean-deviation}, we define the following ``mis-normalized'' empirical mean $\boldsymbol {\hat \mu}$
\begin{align*}    
\boldsymbol {\hat \mu}
:= \frac{ \tilde B }{ \bar B } \vec {\boldsymbol{\mu}}
= \frac{1}{\bar B |S|} \lp( \sum_{(\vec x, y) \in S} \mathbbm 1\{y = 1\} \x \rp).
\end{align*}
We note that $\boldsymbol {\hat \mu}$ and $\boldsymbol {\mu}$ are close to each other:
\begin{align*}
\boldsymbol {\mu} - 
\boldsymbol {\hat \mu}
= \lp( 
1 - \frac{\tilde B}{\bar B} \rp) \;
\frac{1}{\tilde B |S|} \lp( \sum_{(\vec x, y) \in S} \mathbbm 1\{y = 1\} \x  \rp)
= 
O( \opt / B ) \boldsymbol {\mu} \, ,
\end{align*}
% \nnote{did you define $\tilde \mu ?$ maybe hat?}
where the last equality follows from that $\abs{\tilde B - \hat B} \leq \opt$ and $\tilde B > \Omega(B)$.
Since \Cref{lem:mean-distance-bound} implies that $\snorm{2}{\boldsymbol{\mu}} \leq O \lp( \sqrt{\log(1/\opt)} \rp)$ and $B \gtrsim \sqrt{\opt} \log^2(1/\opt)$ by assumption,
it follows that
\begin{align}
\label{eq:mis-normalized-deviation} 
\vec v^\ast \cdot\lp( \boldsymbol {\hat \mu}
- \boldsymbol {\mu} \rp) \leq \tilde O\lp( \sqrt{\opt} \rp).
\end{align}
We then bound from above the distance between
$\vec v^\ast \cdot  \boldsymbol {\hat \mu}$ and $\vec v^\ast \cdot \boldsymbol {\bar \mu}$.
We have that
\begin{align}
\vec v^\ast \cdot \boldsymbol {\hat \mu}
-
\vec v^\ast \cdot \boldsymbol  {\bar \mu}
&=
\frac{1}{\bar B |S|}
\sum_{ (\x, y) \in S }
\lp( \mathbbm 1 \{ h(\x) = +1 \} - \mathbbm 1 \{y = +1 \} \rp) \; \vec v^\ast \cdot \x \nonumber \\
&\leq
\frac{1}{\bar B |S|}
\sqrt{
\sum_{ (\x, y) \in S } 
\lp( \mathbbm 1 \{ h(\x) = +1 \} - \mathbbm 1 \{y = +1 \} \rp)^2 \; 
{\vec v^\ast}^{\top} 
\lp( \sum_{(\x, y) \in S}
\vec x \vec x^{\top} \rp) \vec v^\ast
} \nonumber \\
& \leq 
\frac{1}{\bar B} \; 
O \lp( \sqrt{ \opt   } \rp)
< O \lp( \frac{1}{ \log(1/B) } \rp) \, ,
\label{eq:pure-outlier-mean-deviation}
\end{align}
where in the first inequality we use Cauchy's inequality, in the second inequality we use the fact that there are at most $\opt$ fraction of points with $h(\x) \neq y$, and that the empirical second moments of $S$ is bounded by $2 \vec I$, and in the last inequality we use $\tilde B \gtrsim \sqrt{\opt} \log(1/\opt)$.
Combining \Cref{eq:mis-normalized-deviation} and \Cref{eq:pure-outlier-mean-deviation} then concludes the proof of \Cref{eq:outlier-mean-deviation}.
% Note that if there were no label noises, then $\vec v^\ast \cdot \boldsymbol{\mu}_+$ would be at least $\vec v^\ast \cdot \vec v = t^\ast$.

% Since we certify the overall covariance of $S$ is at most $2 \vec I$ (Line~\ref{line:cov-verify}), 
% it follows that the covariance of the $+1$ points is at most $4 \vec I/ B$.
% Since by our assumption $ \opt \leq B^2$,
With \Cref{eq:outlier-mean-deviation} in our hand, 
we are ready to finish the proof. 
Note that $\vec v$ is the intersection point between the separating hyperplane of $h$ and the line along $\boldsymbol \mu$. 
That means that for some $\lambda\in \R$, we have $\vec v=\lambda \boldsymbol\mu$. 
If $\lambda > 1$, we immediately have that $\|\vec v\|_2 \leq \|\boldsymbol \mu\|_2 \leq \|\boldsymbol \mu \|_2 + O(1/\sqrt{\log(1/B)})$.
Hence, we can focus on the case where $\lambda < 1$.
% and hence using that $\vec v\cdot\vec v^\ast=-t^\ast=|t^\ast|$, we have that $\boldsymbol \mu_+\cdot\vec v^\ast=|t^\ast|/\lambda$. 
% Furthermore, note that $\vec v^\ast\cdot \bar{\boldsymbol\mu}_+\geq |t^\ast|$, hence $\lambda\in(0,1)$ for the noiseless $\bar{\boldsymbol{\mu}}_+$ and from \Cref{eq:outlier-mean-deviation} one can deduce that $\lambda$ is $(1+o(1))$ away from the optimal's one. 
By a simple geometric argument (illustrated in \Cref{fig:angle-proportion}), we have the equality
$$
\frac{\snorm{2}{ \vec v} - \snorm{2}{ \boldsymbol{\mu} }}
{  \snorm{2}{ \vec v } }
=\frac{ \vec v^\ast \cdot (\vec v - \boldsymbol{\mu} )   }{  |t^\ast| }.
$$
By \Cref{eq:outlier-mean-deviation}, we have that the numerator on the right hand side is at most $O(1/\log(1/B))$.
By the assumption of the lemma, we have $|t^*| \geq 10$.
Note that $\vec v^\ast\cdot(\vec v-\boldsymbol \mu)= |t^\ast|(\lambda-1)/\lambda$,
% Hence, if we now 
% our estimation is off from $\boldsymbol{\mu}_+$ by at most
% $2 \; \sqrt{\opt} / \sqrt{B} \leq 2 \; B^{3/2}  $, it follows that $\vec v^\ast \cdot \boldsymbol{\mu}_+ 
% \geq \vec v^\ast \cdot \vec v - 2B^{3/2}$ even under $\opt$ amount of label noise.
hence
\begin{align*}
\frac{\snorm{2}{ \vec v} - \snorm{2}{ \boldsymbol{\mu} }}
{  \snorm{2}{ \vec v } }
=\frac{\lambda\snorm{2}{ \boldsymbol \mu} - \snorm{2}{ \boldsymbol{\mu} }}
{  \lambda\snorm{2}{ \boldsymbol{\mu} } } 
=\frac{ \vec v^\ast \cdot (\vec v - \boldsymbol{\mu} )   }{  t^\ast }
\leq O(  1 / \log(1/B) ) \, ,
\end{align*}
where the first equality is illustrated in \Cref{fig:angle-proportion}, and the last inequality follows from \Cref{eq:the-right-one} and $t^\ast > 10$.
\begin{figure}
    \centering
    \includegraphics[width=10cm]{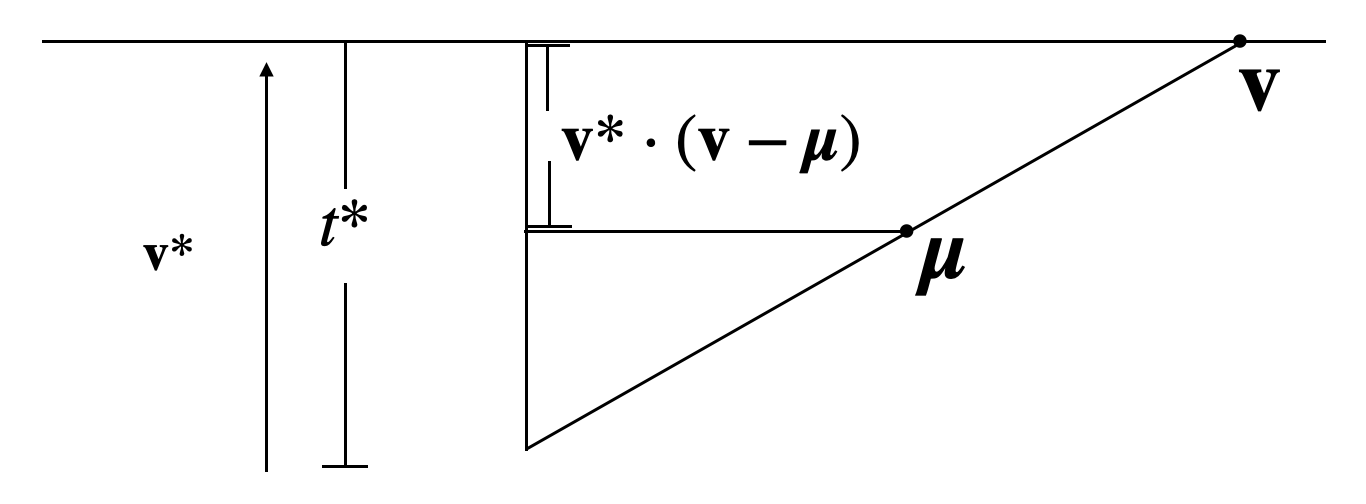}
    \caption{Bound $\snorm{2}{\vec v} - \snorm{2}{\boldsymbol{\mu}}$ via $\vec v^\ast \cdot (\vec v - \boldsymbol{\mu})$.}
    \label{fig:angle-proportion}
\end{figure}

Therefore, it follows that
$$
\snorm{2}{\vec v} 
\leq \snorm{2}{\boldsymbol{\mu}} \lp( 1 +O(1 / \log(1/B)) \rp).
$$
By \Cref{lem:mean-distance-bound}, we must have
$\snorm{2}{\boldsymbol{\mu}} \leq O( \sqrt{ \log(1/B) } )$. We therefore have
$$
\snorm{2}{\vec v} \leq \snorm{2}{\boldsymbol{\mu}} + O( \sqrt{\log(1/B)} / \log(1/B) ) \leq 
\snorm{2}{\boldsymbol{\mu}} + O(1/ \sqrt{ \log(1/B) } ).
$$
This concludes the proof of \Cref{lem:mu-v-distance}
\end{proof}

\subsection{Localization Center Search for Halfspaces with Constant Thresholds}
\begin{lemma}
\label{lem:small-offset-center-search}
Let $D$ be a distribution over $\R^d \times \{\pm 1\}$ and $h(\x) = \sgn(\vec v^\ast\cdot \x + t^\ast)$ be a halfspace that achieves $\opt$ error with respect to $D$.
% Let $S$ be a set of $N$ \iid samples from $D$ where $N = \poly(d/\eps)$. 
Then, there exists an algorithm that takes $N:=\poly(d/\eps)$ many samples, and runs in time $\poly(N)$ and with high constant probability, 
either
% satisfying the following:
% with , it holds 
\begin{enumerate}
    \item reports violation of the distribution assumption, in which case the report is correct.
    \item 
    returns a list of at most $O(1/\eps^2)$ points.
    In addition, if it holds that $|t^\ast| \leq 10$, the list contains at least one $(\eps^2,   \Theta(1) )$-good localization center, where $C>0$ is a universal constant.
\end{enumerate}
\end{lemma}
\begin{proof}
Let $S$ be the empirical distribution over $N:=\poly(d/\eps)$ samples.
Let $\eta > 0$ be some sufficiently small constant.
The algorithm first verifies that the  degree up to $k = \Theta(\log(1/\eta) / \eta^2)$ moments of the empirical distribution $S$ match with those of $\normal(\vec 0, \vec I)$ up to an additive error of $\Delta = \frac{1}{kd^k} \lp( 
\frac{1}{C \sqrt{k}}
\rp)^{k+1}$.
Then it computes the Chow parameter
$\vec w := \E_{(\x, y) \sim S}[ y \x]$.
Lastly, it returns the list of points
$ i \; \eps^2 \; \vec w / \snorm{2}{\vec w} $ for $i \in [ 10/ \eps^2 ]$.

The soundness of the test follows from standard concentration inequalities of low degree moments of $\normal(\vec 0, \vec I)$.
We hence omit the proof for conciseness.

We now argue that,  with high constant probability, when the algorithm does not reject, we find a good localization center.
Let $\vec v$ be the intersection between the line along $\vec w$ and the halfspace.
We will argue that $\snorm{2}{\vec v}$ is bounded from above by some constant, from which the existence of a good localization center in the output list follows.
Since we assume that $t^\ast \leq 10$, it suffices to show that the angle between $\vec w$ and $\vec v^\ast$ is bounded from above by some sufficiently small constant.

Using Lemma 2.3 of \cite{diakonikolas2023efficient}, we have
that 
\begin{align}
\label{eq:moment-match}
\E_{ (\x, y) \sim S }
\lp[ 
h(\x) \x
\rp] 
- 
\E_{ \x \sim \normal(\vec 0, \vec I) }
\lp[  h(\x) \x \rp]
\leq O\lp( \sqrt{\eta} \rp).    
\end{align}
From Lemma 4.3 of \cite{DKS18a}, it holds that
\begin{align}
\label{eq:chow-equality}    
\E_{ \x \sim \normal(\vec 0, \vec I) }
\lp[  h(\x) \x \rp]
= G(t^\ast) \; \vec v^\ast \, ,
\end{align}
where $G(t^\ast)$ is the probability density function of the one dimensional standard normal.
Finally, let $\vec u$ be an arbitrary unit vector in $\R^d$.
Since we verify that the eigenvalues of the empirical covariance of the samples is bounded from above by $2$, we have
\begin{align}
\label{eq:cov-error-bound}
\E_{ (\x, y) \sim S }
\lp[ 
y \x \cdot \vec u
\rp] - 
\E_{ (\x, y) \sim S }
\lp[ 
h(\x) \x \cdot \vec u
\rp] 
\leq \sqrt{ 
\E_{ (\x, y) \sim S }
\lp[ \lp( h(\x) - y \rp)^2 \rp]
\; \E_{ (\x, y) \sim S } \lp[ \vec u^T \x \x^T \vec u \rp]
}
\leq O \lp( \sqrt{\opt} \rp) \, ,
\end{align}
where the first inequality follows from Cauchy's inequality, and the second inequality follows from that $h$ has error at most $\opt$ and $S$ has its covariance bounded by $2 \vec I$.
Let $\vec w$ be the Chow parameter estimated.
Combining \Cref{eq:moment-match,eq:chow-equality,eq:cov-error-bound}, we have that
$
\angle \lp( \vec w , \vec v^\ast  \rp)
\leq O(\sqrt{\eta})$, which can be made sufficiently small if we choose $\eta$ to be some sufficiently small constant.
This then concludes the proof of \Cref{lem:small-offset-center-search}.
\end{proof}

\subsection{Proof of \texorpdfstring{\Cref{lem:mean-distance-bound}}{Lg}}
Before presenting the analysis, we first present some useful bounds related to the standard Gaussian CDF function $\Phi(t) := 
\Pr_{ x \sim \normal(0, 1) } \lp[ 
x > t
\rp].
$
\begin{claim} \label{clm:gaussian-cdf-bound}
Let $x > 10$, and $b \leq O(x^{-1})$.
Then it holds $\Phi(x + b) \geq \Omega(1) \Phi(x)$.
\end{claim}
\begin{proof}
We have
\begin{align*}
    \Phi( x+b )
    &= 
    \frac{1}{\sqrt{2 \pi}}
    \int_{t = x+b }^\infty
    \exp( -t^2/2 ) \\
    &=
    \frac{1}{\sqrt{2 \pi}}
    \int_{t = x  }^\infty
    \exp( - (t + b ) )  ^2/2 )\\
    &\geq 
    \frac{1}{\sqrt{2 \pi}}
    \int_{t = x  }^{10 \; x }
    \exp( - (t + b ) )  ^2/2 )\\
    &\geq 
    \frac{1}{\sqrt{2 \pi}}
    \int_{t = x  }^{10 \; x }
    \exp( -t^2/2 - O(1) ) \\
    &= 
    \Omega(1) \; 
    \lp(
    \Phi(  x )
    - 
    \Phi(  10 \; x ) \rp) \geq \Omega(1)  \; \Phi(  x ) \, ,
\end{align*}
where the first inequality follows from the positivity of the $\exp(\cdot)$ function, and the second inequality follows from that 
$t b, b^2 = O(1)$ since 
$b \leq O(1/x) \leq O(1/t)$.
\end{proof}
We now give the proof of \Cref{lem:mean-distance-bound}.
\begin{proof}[Proof of \Cref{lem:mean-distance-bound}]
Recall that $S$ is a set of $N$ \iid samples from $D$ where $N = \poly(d, 1/\eps)$,  
$\boldsymbol{\mu}_+$ is the empirical mean of the samples in $S$ with label $+1$, $H$ is the resulting set after projecting the points in $S$ along the direction of $\boldsymbol{\mu}_+$, 
and $\tilde B$ is the fraction of points in $H$ (or $S$) with label $+1$.
For convenience, we define $ u := \snorm{2}{ \boldsymbol{\mu}_+ }$.
As a slight abuse of notation, we also denote by $H$ the empirical distribution over the set $H$.

In our analysis, it is easier to bound $\Phi(\snorm{2}{\vec u_+})$ in terms of $\tilde B$ thatn to bound it in terms of $B$. 
Nonetheless, we remark that $B$ and $\tilde B$ are approximately the same.
In particular, using standard concentration inequalities, we have that $\tilde B \eqdef \Pr_{ (\x, y) \sim H } \lp[  y = 1 \rp]$ should satisfy
$\abs{ \tilde B - B } \leq \eps/10$.
By our assumption that $\eps < B^2 $, we have that that $\tilde B > B / 2$, which implies that 
$\tilde B / \log(1/ \tilde B ) > (B/2) / \log(2/B) = \Omega( B / \log(1/B))$, since $x /\log(1/x)$ is monotonically increasing for $x \in [0,1)$.
Hence, it suffices to show that
\begin{align} \label{eq:significant-right-mass}
\Phi(\snorm{2}{\boldsymbol{\mu}_+})
\geq  \Omega\lp(  \tilde B / \log(1/\tilde B) \rp).
\end{align}

We further introduce the following sets to facilitate the discussion. We divide $H$ into 
two sets $M, P$, where $M$ is the set 
of points with label $-1$, and
$P$ is the set of points with label $+1$.
Among $P$, we further divide into $P_L$ and $P_R$, where $P_L := \{ x \in P \text{ such that } x \leq u
\}$, and $P_R := \{ x \in P \text{ such that } 
x > u \}$. We proceed to examine two cases separately.

\textbf{Case I}
In the first case, 
suppose more than half of the points from $P$ are within the range $[u - 1 / \sqrt{\log(1/\tilde B)}, \infty]$.
Since $\abs{P} / \abs{H} = \tilde B$, we have
$$
% \Pr_{ (\x, y) \sim S }\lp[ \frac{\x \cdot  \boldsymbol{\mu}_+}{\snorm{2}{\boldsymbol{\mu}_+}}
% > u - 1/\log(1/\tilde B) \rp]
% \geq 
\Pr_{ x \sim H }\lp[ 
x > u - 1/\sqrt{\log(1/\tilde B)}
\rp]
\geq
\Pr_{ x \sim H }\lp[ 
  x \in P \text{ and }
  x > u - 1/\sqrt{\log(1/\tilde B)} 
\rp]
\geq \tilde B/2.
$$
Recall that in Line~\ref{line:cdf-check} we explicitly verify that $H$ is close to $\normal(0, 1)$ in Kolmogorov distance.
It follows that
\begin{align*}
\Phi(u - 1/\sqrt{\log(1/\tilde B)})
\geq 
\Pr_{ x \sim H }\lp[ 
x > u - 1/\sqrt{\log(1/\tilde B)})
\rp] - \eps
\geq \Omega(\tilde B).
% \geq \Pr_{ x \sim H }\lp[ 
% x > u )
% \rp] 
%  - 1/log(1/\tilde B)
\end{align*}
This allows us to bound $u - 1/\sqrt{\log(1/\tilde B)}$ from above by $O\lp( \sqrt{ \log(1/\tilde B) } \rp)$.
Therefore, applying Claim~\ref{clm:gaussian-cdf-bound} gives that
$\Phi(u ) \geq \Omega(1)\Phi\lp(u - 1/\sqrt{\log(1/\tilde B)}\rp) \geq \Omega(\tilde B)$, which implies \Cref{eq:significant-right-mass}.

\textbf{Case II}
In the second case, we have more than half of the points from $P$ lying in the range $[-\infty, u - 1/\sqrt{\log(1/\tilde B)}].$
In this case, we can bound from below the contribution from the points in $P_L$ to the mean of $P$ by
\begin{align}
\label{eq:left-deviation-bound}    
\end{align}
$$
\frac{1}{|P|} \sum_{x \in P_L} (u - x)
\geq 
\frac{1}{|P|}
\sum_{x \in P_L } \mathbbm 1\lp\{ x < u - \frac{1}{\sqrt{\log(1/\tilde B)}}\rp\} \frac{1}{\sqrt{\log(1/\tilde B)}}
\geq \frac{1}{2 \sqrt{\log(1/\tilde B)}}.
$$
This then implies that the points in $P_R$ need to satisfy that
\begin{align} \label{eq:right-deviation-bound}
\sum_{x \in P_R} x
\geq 
\sum_{x \in P_R} (x-u)
= 
|H| \; \frac{|P|}{|H|} \; 
\frac{1}{|P|}
\sum_{x \in P_L} (u - x)
\geq |H| \frac{\tilde B}{ 2 \sqrt{\log(1/\tilde B)}} \, ,
\end{align}
where the first inequality holds because $u := \snorm{2}{\boldsymbol{\mu}_+} \geq 0$, and
the equality holds since
$
\sum_{x \in P_R} (x - u) = \sum_{x \in P_L} (u - x)
$, and the last inequality follows from the definition $\tilde B := |P| / |H|$ and \Cref{eq:left-deviation-bound}. 
For the sake of contradiction, assume that $ 
\Pr_{x \sim H}[ x \in P_R ] \leq c \; \tilde B / \log(1/\tilde B)$ for some sufficiently small constant $c>0$.
Then, \new{conditioned} on the event that Line~\ref{line:stability-check} passes the test; if we remove all points from $P_R$, the mean deviates by at most $O \lp( c \; 
\frac{ \tilde B  }{\log(1/B)} \; 
\sqrt{\log\lp( \frac{ \sqrt{\log(1/B)}}{\tilde B}  \rp)} \rp)
= O \lp( c \; \tilde B / \sqrt{\log(1/B)} \rp)$.
In particular, the mean of the remaining samples must satisfy
\begin{align} \label{eq:removed-mean}
\abs{\sum_{ x \in H\backslash P_R} x }
% \leq \frac{1}{|H \backslash P_R|}    
% \abs{\sum_{ x \in H\backslash P_R}
% x}
\leq O \lp( c \; \tilde B / \sqrt{\log(1/B)} \rp) \; |H \backslash P_R|
\leq O \lp( c \; \tilde B / \sqrt{\log(1/B)} \rp) \; |H|.
\end{align}
\new{Conditioned on the event that }Line~\ref{eq:mean-deviation} passes, we have that
\begin{align} \label{eq:original-mean}
\abs{    \sum_{ x \in H} x }
= \eps \; |H|.
\end{align}
Note that
$
\sum_{ x \in P_R} x 
= 
\sum_{ x \in H} x
+ \lp(\sum_{ x \in H \backslash} -x \rp).
$ Thus, by using 
\Cref{eq:removed-mean,eq:original-mean}, and the triangle inequality, we have that
\begin{align}
\label{eq:deviation-mass}    
\abs{\sum_{ x \in P_R} x }
\leq O\lp(c \frac{\tilde B}{\sqrt{\log(1/\tilde B)}} + \eps \rp) \abs{H}.
\end{align}
Recall that $\eps \lesssim \tilde B$ by our assumption.
Thus, by choosing $c$ to be a sufficiently small constant, 
\Cref{eq:deviation-mass} contradicts \Cref{eq:right-deviation-bound}.
Therefore, we have that 
$
\Pr_{x \sim H}[x > u] = 
\Pr_{x \sim H}[ x  \in P_R ] > \Omega\lp( \tilde B / \log(1/\tilde B) \rp)
$.
With similar arguments as in Case I, conditioned on the event that Line~\ref{line:cdf-check} passes, we then have that
$$
\Phi(\snorm{2}{\boldsymbol{\mu}_+})
\geq 
\Pr_{ \x \in H} \lp[ x > \snorm{2}{\boldsymbol{\mu}_+} \rp] - \eps
\geq \Omega\lp( \tilde B / \log(1/\tilde B) \rp)\;.
$$
Therefore \Cref{eq:significant-right-mass} holds
for Case II, and this concludes the proof of \Cref{lem:mean-distance-bound}.
\end{proof}

% \subsection{Proof of }

% \subsection{Proof of }

\section{Proof of Main Theorem (\texorpdfstring{\Cref{thm:testable-learning-agnostic}}{Lg})}
\label{app:23}
\begin{Ualgorithm}[h]
	\centering
	\fbox{\parbox{6in}{
			{\bf Input:} Sample access to a distribution $D$ over $\R^d \times \{\pm 1\}$;  
   tolerance $\eps$; failure probability $\tau$.\\
   {\bf Output:} \new{Reject or output a halfspace $h$, with error at most $\tilde O(\sqrt{\opt}+\eps)$.}
\begin{enumerate}
\item Run \text{Find-Localization-Center} (\Cref{alg1:localization-center}) to find a set of localization centers $\text{Localization-Centers} := \{\vec v^{(i)}\}_{i=1}^{\log(1/B)/\eps^2}$
\item Initialize an empty set $S$ for storing candidate halfspaces.
\item For $\vec w \in \text{Localization-Centers}$
\begin{enumerate}
    \item Set $\sigma := \min \lp( \snorm{2}{\vec w}^{-1}, \sqrt{1/2} \rp)$.
    \item Define $G$ as the distribution 
    over $\R^d \times \{\pm 1\}$
    obtained by performing $(\vec w, \sigma)$-rejection sampling on the $\x$-marginal of $D$.
    \item Define $D'$ as the distribution obtained by applying
    the transformation $h(\x) = \vec \Sigma^{-1/2} (\x - \vec w) $ on the $\x$-marginal of $G$, where $\vec \Sigma := \vec I - (1 - \sigma^2) \vec w \vec w^T / \snorm{2}{\vec w}^2$.
    \item Run \text{Almost-Homogeneous-Testable-Learn} from \Cref{prop:small-offset-halfspace} with sample access to $D'$, and denote the learned halfspace vector as $\vec v$.
    \item Run \text{Wedge-Bound} (\Cref{alg:wedge-bound}) with sample access to $D$, and the unit vector $\vec \Sigma^{-1/2} \vec v$.
    \item For $i = -\ceil{\log(1/\eps)/\eps}, \cdots, \ceil{\log(1/\eps)/\eps}$,  add the halfspace $h(\x) = \sgn\lp( \vec \Sigma^{-1/2} \vec v + i \eps \rp)$ to $S$.
\end{enumerate}
% \item Draw $N$ \iid samples from $D$, where $N = \poly(d/\eps)$.
\item Compute the empirical errors of the halfspaces within $S$ using $\poly(d/\eps)$ \iid samples from $D$.
\item Return the halfspace with the smallest empirical error.
\end{enumerate}
}}
\vspace{0.2cm}
\caption{Testable-Learn-General-Halfspace} 
\label{alg:main}
\end{Ualgorithm}

\begin{proof}[Proof of Main Theorem]
Our algorithm produces a list $S$ of halfspaces and picks the one with the smallest empirical error.
It suffices to show that at least one of the halfspaces contained in the list $S$ achieves error at most $\tilde O(\sqrt{\opt} + \eps)$. Let $B$ be the mass of the minority label.
Our list always contains a hypothesis that outputs either the label $+1$ or $-1$ for any point, therefore if $B < \max(\sqrt{\eps}, \sqrt{\opt})$ then one of these two halfspaces obtains error at most $\sqrt{\opt}+\sqrt{\eps}$.
% If , then we are fine as long as the remaining part of the algorithm does not falsely reject the distribution.
Hence, we can assume $B > \max(\sqrt{\eps}, \sqrt{\opt})$, for the rest of the analysis of the algorithm.

Let $h(\x) = \sgn(\vec v^\ast \cdot \x + t^\ast)$ be the optimal halfspace.
If $t^\ast > 10$, we can apply \Cref{lem:certify-localization-center} to obtain a list of candidate localization centers with the guarantee that at least one of them is an $ ( \eps^2, \Omega(B/\log(1/B)) ) $-good localization center.
If $t^\ast < 10$, we directly estimate the Chow vector (enhanced with moment matching) which, up to normalization, gives us a constant approximation to the defining vector $\vec v^\ast$ of the optimal halfspace (see \Cref{lem:small-offset-center-search}). Then, we guess the intersection between the Chow vector and the halfspace $h$ up to accuracy $\eps^2$. This gives us another list containing at least one point that is $(\eps^2, \Theta(1))$-good localization center.
The list $L$ in both cases will have size at most $\poly(1/\eps)$.

For each $\vec w\in L$, we apply $(\vec w, \sigma)$ rejection sampling with $\sigma = \min( \sqrt{1/2}, \snorm{2}{\vec w}^{-1} )$.
The resulting distribution is then $\normal(\vec w, \vec \Sigma)$, where
$\vec \Sigma = \vec I - (1 - \sigma^2) \vec w \vec w^T / \snorm{2}{\vec w}^2$.
Next we apply the appropriate transformation to make the $\x$-marginal isotropic, and call this distribution $D_{\vec w}'$.
By assumption, there exists at least one $\vec w\in L$ that is $ ( \eps^2, \Omega(B/\log(1/B)) )$-good localization center. Let $\vec w$ be an $ ( \eps^2, \Omega(B/\log(1/B)) )$-good localization center. \new{From \Cref{lem:good-center}, after the transformation, an optimal halfspace of $D_{\vec w}'$ is a halfspace $\tilde h$  with offset $O(1) \; \eps^2 \log( 1/B ) \leq \eps$. The acceptance probability of a sample in the distribution $D_{\vec w}'$ is at least $\Omega(B/\log(1/B))$.}
It follows that the error of the halfspace under the new distribution is at most $O(\opt \; \log(1/B) / B ) < \tilde O(\sqrt{\opt})$.

Let $\tilde h(\x) = \sgn( \tilde {\vec v}\cdot \x + \tilde t )$ be the halfspace after applying the transformation that makes the new distribution isotropic.
Using \Cref{prop:small-offset-halfspace}, we can then learn a vector $\vec v$ satisfying
$$
\snorm{2}{\vec v - \tilde {\vec v}}
\leq \tilde O(\sqrt{\opt} + \eps).
$$
Notice that by \Cref{lem:good-center}, $\tilde {\vec v}$ is parallel to the vector $\vec \Sigma^{1/2} \vec v^\ast$.
We argue the estimation $\vec v$ after reverting the transformation $\vec \Sigma^{1/2}$ gives us a good approximation of $\vec v^\ast$.
To do so, we need to apply \Cref{lem:transformed-angle}.
A crucial requirement of \Cref{lem:transformed-angle} is that the angle between the localization direction $\vec w$ and $\vec v^\ast$ must not be too large. To show this, we need to consider two cases depending on whether we have $\abs{t^\ast} <10$.

In the case that $\abs{t^\ast} < 10$, the angle between $\vec w$ and $\vec v^\ast$ will be bounded by a small constant (since the direction of $\vec w$ is obtained through Chow parameter estimation in this case).

The harder case is when $\abs{t^\ast} > 10$.
In this case, we write $\vec w = b \vec v^\ast + a \vec z$ where $\vec z$ is orthogonal to $\vec v^\ast$. 
Since the distance from $\vec w$ to the halfspace is at most $\eps^2$, the component of $\vec w$ along $\vec v^\ast$ must be at least $(10 - \eps^2)$, i.e., $b > (10 - \eps^2)$.
On the other hand,
since $\vec w$ satisfies $\Phi(\snorm{2}{\vec w}) > \Omega(B/\log(1/B))$, it follows that 
$\snorm{2}{\vec w} \leq O(\sqrt{\log(1/B)})$, which further implies that $a < O(\sqrt{\log(1/B)})$. 
Combining the two observations we have that $a/b \leq O(\sqrt{\log(1/B)})$.
Let $x = b/a$, we have that
\begin{align*}
\snorm{2}{\vec v^\ast - \vec w / \snorm{2}{\vec w}}^2
&= \lp( 1 - \frac{1}{\sqrt{1 + x^2}}\rp)^2
+ \lp( \frac{x}{\sqrt{1+x^2}} \rp)^2
= 1 + \frac{1}{1+x^2} - \frac{2}{\sqrt{1+x^2}}
+ \frac{x^2}{1 + x^2} \\
&= 2 - \frac{2}{\sqrt{1+x^2}}
\leq 2 - \frac{2}{ O(\sqrt{\log(1/B)}) }.    
\end{align*}

For this particular localization center $\vec w$, we also have $\sigma \geq \Omega(1/\sqrt{\log(1/B)}) \geq \Omega(1 / \sqrt{\log(1/\opt)})$ by \Cref{lem:good-center}.

We can therefore apply \Cref{lem:transformed-angle} with 
$\beta = \sqrt{2 - \frac{2}{ O(\sqrt{\log(1/B)}) }}$,
$\delta = \tilde O(\sqrt{\opt}) + \eps$
and $\sigma = \Omega(1/\sqrt{\log(1/\opt)})$.
Note that
$$
\frac{\beta}{1 - \beta^2/2}
\leq O(\sqrt{\log(1/B)}).
$$
This implies that, after reverting the transformation, we obtain a vector $\vec v$ satisfying that
\begin{align*}
\snorm{2}{\vec v - \vec v^\ast}
&\leq 
O(1) \; (\sigma + \beta) \;
\lp( \delta \sigma^{-1} \frac{\beta}{1 - \beta^2/2} + \delta  \rp) \\
&\leq 
O(\delta) \; \lp( \sqrt{\log(1/\opt)} \; \sqrt{\log(1/\opt) } \rp) 
\leq \tilde O(\sqrt{\opt} + \eps).
\end{align*}
We then guess the offset of the optimal halfspace up to accuracy $\eps$.
Denote the best guess as $\bar t$.
Applying \Cref{lem:general-wedge-bound}, we have that 
$$
\Pr_{(\x, y) \sim D}[ \sgn(\vec v \cdot \x + \bar t) \neq \sgn(\vec v^\ast \cdot \x + t^\ast)  ]
\leq O(1) \; 
\lp( \snorm{2}{\vec v - \vec v^\ast}
+ \abs{\bar t - t^\ast}
\rp)
\leq \tilde O(\sqrt{\opt} + \eps).
$$
The above concludes the proof of \Cref{thm:testable-learning-agnostic} for constant failure probability $\tau$.
In order to boost the success probability, we will run the algorithm $10 \; \log(1/\tau)$ many times. If the algorithm outputs reject for more than half of the time, we output reject. Otherwise, we collect the halfspaces learned, draw $O \lp(  d \log\log(1/\tau) / \eps^2 \rp)$ many \iid samples, and output the halfspace that achieves the smallest empirical error.
If $D$ truly has marginal $D_\x$, each invocation of the algorithm outputs accepts with high constant probability. Since we only output reject when the algorithm outputs reject in more than half of the trials, it follows that we reject with probability at most $\tau$.
If we do not output reject, we must have collected at least $5 \log(1/\tau)$ many halfspaces, and each halfspace achieves the required learning error with high constant probability. So there exists a halfspace in the learned list with error $\wt O(\sqrt{\opt}) + \eps$ with probability at least $1 - \tau/2$.
Besides, with probability at least $1 - \tau/2$, we can estimate the error of these halfspaces up to accuracy $\eps$. 
Conditioned on the above two events, which hold simultaneously with probability at least $1 - \tau$ by the union bound, it then follows that the final halfspace picked will have error at most $\wt O(\sqrt{\opt}) + 2 \eps$.
This concludes the proof of \Cref{thm:testable-learning-agnostic}.
\end{proof}
\end{document}